\documentclass[twoside,11pt]{article}

\usepackage{amsmath, amssymb, amsfonts, amsthm, mathtools, mathrsfs, bbm} 
\usepackage{graphicx} 
\usepackage{hyperref}
\usepackage{geometry} 
\usepackage{setspace} 
\usepackage{caption} 
\usepackage[numbers, sort]{natbib}
\usepackage{lineno} 

\usepackage{lipsum}
\usepackage{titlesec}
\usepackage{fancyhdr} 
\usepackage{array}
\usepackage{xcolor}
\usepackage{enumitem}
\usepackage{tikz} 
\usetikzlibrary{calc, arrows.meta, intersections, patterns, positioning, shapes.misc, fadings, through,decorations.pathreplacing, trees, shapes.geometric,math,arrows}
\usepackage{tikz-cd}
\usepackage{amscd}
\usepackage{pgfplots}
\pgfplotsset{compat=1.18}
\usepgfplotslibrary{fillbetween}
\usepackage{url}
\usepackage{aliascnt}
\usepackage{subcaption}

\newcommand{\supp}[1]{\mathrm{supp}\!\left(#1\right)}
\newcommand{\suppinline}[1]{\mathrm{supp}(#1)}
\newcommand{\spanC}[1]{\mathrm{span}_{\mathbb{C}}\!\left(#1\right)}
\newcommand{\spanR}[1]{\mathrm{span}_{\mathbb{R}}\!\left(#1\right)}
\newcommand{\spanRinline}[1]{\mathrm{span}_{\mathbb{R}}(#1)}
\newcommand{\spanCinline}[1]{\mathrm{span}_{\mathbb{C}}(#1)}

\newcommand{\dimC}[1]{\mathrm{dim}_{\mathbb{C}}\!\left(#1\right)}

\newcommand{\dimR}[1]{\mathrm{dim}_{\mathbb{R}}\!\left(#1\right)}
\newcommand{\dimRinline}[1]{\mathrm{dim}_{\mathbb{R}}(#1)}
\newcommand{\sepcap}[1]{\mathcal{SC}\!\left(#1\right)}

\newcommand{\rank}[1]{\mathrm{rank}\!\left(#1\right)}
\newcommand{\rankinline}[1]{\mathrm{rank}(#1)}
\newcommand{\sepcaps}[2]{\mathcal{SC}^{#1}\!\left(#2\right)}
\newcommand{\sepcapsinline}[2]{\mathcal{SC}^{#1}(#2)}
\newcommand{\innerprod}[2]{\left\langle #1, #2 \right\rangle}

\newcommand{\Mod}[1]{\ (\mathrm{mod}\ #1)}
\newcommand\restr[2]{{
  \left.\kern-\nulldelimiterspace 
  #1 
  \vphantom{\big|} 
  \right|_{#2} 
  }}
\newcommand\restrinline[2]{{
  \,\kern-\nulldelimiterspace 
  #1 
  |_{#2} 
  }}
\newcommand{\restrmeasure}[2]{#1 \llcorner #2}
\newcommand{\pushfwd}[2]{#2_{\sharp}\!\left(#1\right)}
\newcommand{\pushfwdinline}[2]{#2_{\sharp}(#1)}
\newcommand{\pushfwdnb}[2]{#2_{\sharp}#1}
\newcommand{\essinf}[2]{\mathrm{ess}\,\mathrm{inf}_{#1}\,#2}
\newcommand{\lip}[1]{\mathrm{Lip}\!\left(#1\right)}
\newcommand{\trace}[1]{\mathrm{Tr}\!\left(#1\right)}
\newcommand{\traceinline}[1]{\mathrm{Tr}(#1)}
\newcommand{\lipbound}{t}
\renewcommand{\ker}[1]{\operatorname{ker}(#1)}
\newcommand{\diag}[1]{\operatorname{diag}(#1)}

\newtheorem{theorem}{Theorem}[section]

\newaliascnt{lemma}{theorem}
\newtheorem{lemma}[lemma]{Lemma}
\aliascntresetthe{lemma}

\newaliascnt{proposition}{theorem}
\newtheorem{proposition}[proposition]{Proposition}
\aliascntresetthe{proposition}

\newaliascnt{assumption}{theorem}

\aliascntresetthe{assumption}

\newtheorem{corollary}{Corollary}[theorem]

\newaliascnt{definition}{theorem}
\theoremstyle{definition}
\newtheorem{definition}[definition]{Definition}
\aliascntresetthe{definition}

\newaliascnt{remark}{theorem}
\theoremstyle{remark}
\newtheorem{remark}[remark]{Remark}
\aliascntresetthe{remark}

\numberwithin{equation}{section}
\counterwithin{figure}{section}

\titleformat{\section}[block]{\large\bfseries}{\thesection}{1em}{}
\titleformat{\subsection}[block]{\normalsize\bfseries}{\thesubsection}{1em}{}
\titleformat{\subsubsection}[block]{\normalsize\bfseries}{\thesubsubsection}{1em}{}

\titleformat{\paragraph}[runin]
  {\normalfont\normalsize\bfseries}
  {\theparagraph}
  {1em}
  {}
  [.]

\fancypagestyle{plain}{
    \fancyhf{}

}

\newcommand{\mytitle}{Separation Capacity of Scattering Networks on Low-Dimensional Datasets}
\title{\mytitle}
\author{
    Konstantin H\"aberle\\
    ETH Zurich\\
    \texttt{haeberlk@ethz.ch}
    \and
    Helmut B\"olcskei\\
    ETH Zurich\\
    \texttt{hboelcskei@ethz.ch}
}
\date{}

\usepackage[capitalize,nameinlink,compress]{cleveref}
\crefname{equation}{}{}
\crefname{subsection}{Subsection}{Subsections}
\crefname{remark}{Remark}{Remarks}
\crefname{assumption}{Assumption}{Assumptions}

\begin{document}

\maketitle

\begin{abstract}
    We aim to identify scattering network architectures that maximize the separation capacity on data with low intrinsic dimension. 
    The networks we consider employ a fixed monomial nonlinearity and no pooling, so that the only design variable is the frame generated by the network filters.
    For data modeled as rectifiable sets, we first characterize and bound the separation capacity of general feature extractors in terms of the geometry of the dataset. We then particularize to scattering networks and obtain two design criteria: (i) the filters should meet the data on sufficiently many frequencies, and (ii) the matrices coupling the frame to the geometry of the data should be well-conditioned.
    \\[1em]
\end{abstract}

\section{Introduction}
A common intuition for the success of deep learning models in classification and regression tasks is that real-world data, although embedded in high-dimensional spaces, often exhibit an intrinsic low-dimensional structure \cite{bengio2013representation,fefferman2016testing}. In this paper, we model this intrinsic low-dimensional structure using the framework of geometric measure theory \cite{federer2014geometric,ambrosio2000currents}. Specifically, we assume that data lie on a rectifiable set, i.e., a set that, up to Hausdorff measure zero, is covered by a countable union of Lipschitz images of subsets of a Euclidean space. This class contains not only smooth submanifolds but also countable unions thereof, such as the unions of linear subspaces that constitute sparse-signal models.

\sloppy This paper studies the classification capabilities of scattering networks \cite{mallat2012group,wiatowski2017deep,haberle2026scattering} on rectifiable sets, as measured by Cover's separation capacity \cite{cover1965geometrical,haberle2026function}. Scattering networks are multi-layered, neural-network-type architectures in which each layer computes convolutions with frame-generating filters \cite{christensen2003introduction}, followed by pointwise nonlinearities and, in general, pooling operators. We particularize to the pooling-free case with a fixed monomial nonlinearity, so that the design freedom lies entirely in the choice of the frame.
The central question addressed in this paper is how the geometry of the data should enter this choice. Our aim is to characterize the filters that maximize the separation capacity for a given rectifiable dataset.

Two main contributions are reported. 
First, we relate the separation capacity of general feature extractors to the geometry of the underlying rectifiable set. Notably, the separation capacity is bounded below by the rank of the differential on the approximate tangent spaces and bounded above by the rank of a second-moment matrix, which couples the feature extractor to the global geometry of the dataset.
Second, by applying these bounds to scattering networks, we establish filter-design criteria. The upper bounds prescribe that the spectral supports of the filters, intersected with the spectrum of the dataset, should not be contained in a coset of a proper subgroup. The lower bounds, in turn, suggest choosing the filters so that, for each Lipschitz parametrization of the underlying rectifiable set, an associated filter-dependent matrix is as well-conditioned as possible. For sparse signals, the criterion reduces to minimizing the restricted isometry constant of a single matrix.

The paper is organized as follows. \Cref{sec:sc} develops the connection between the separation capacity of general feature extractors on rectifiable sets and the geometry of those sets. In \cref{sec:sc-snet}, we apply these results to scattering networks to establish filter-design criteria, first for sparse signals and then for more general rectifiable sets. \Cref{app:cs} reviews the restricted isometry property used in the sparse-signal lower bound.

\paragraph{Notation} 
We write $\restrmeasure{\mu}{A}$ for the restriction of a measure $\mu$ on a set $X$ to a measurable subset $A$. For a measurable map $\varphi \colon X \to Y$, the pushforward of $\mu$ by $\varphi$ is the measure $\pushfwdnb{\mu}{\varphi}$ on $Y$ given by $(\pushfwdnb{\mu}{\varphi})(B) = \mu(\varphi^{-1}(B))$ for measurable $B \subseteq Y$. When $X$ is a topological space, the support of $\mu$ is $\supp{\mu} = \{x \in X : \mu(U) > 0 \text{ for every open neighborhood } U \text{ of } x\}.$ We write $\mathcal{L}^n$ for the $n$-dimensional Lebesgue measure on $\mathbb R^n$ ($n \in \mathbb N$) and $\mathcal{H}^s$ for the $s$-dimensional Hausdorff measure ($s \geq 0$). Finally, for $\mathbb K \in \{\mathbb R, \mathbb C\}$ and $A \subseteq \mathbb K^n$,
we write $\mathrm{span}_{\mathbb K}(A)$ for the set of finite linear combinations of vectors in $A$ with scalars in $\mathbb K$, and $\dim_{\mathbb K}(V)$ for the dimension of a linear space $V$ over $\mathbb K$.

\section{Separation capacity computations on rectifiable sets}\label{sec:sc}
This section develops the capacity estimates used later for scattering networks. The first step is a subspace characterization of the separation capacity. The subsequent estimates relate that characterization to rectifiability, approximate tangent spaces, and second moments of the pushforward of $\mathcal{H}^s$ by the feature extractor.
We begin by introducing the notion of $s$-separation capacity.
\begin{definition}[$s$-separation capacity, \cite{haberle2026function}]
    Let $E \subseteq \mathbb R^M$ be $\mathcal{H}^s$-measurable with $\mathcal{H}^s(E)>0$ for some $s\geq 0$, and let $\Phi \colon E \to \mathbb R^{M'}$. Denote by $\sepcaps{s}{\Phi}$ the largest $N \in \mathbb N$ such that for $(\mathcal{H}^{s})^N$-a.e. $N$-tuple $F\coloneqq(f_1,\ldots,f_N) \in E^N$ at least $50\%$ of all possible dichotomies of $F$ are $\Phi$-separable. If there is no such $N \in \mathbb N$, set $\sepcaps{s}{\Phi} \coloneqq 0$. We call $\sepcaps{s}{\Phi}$ the \emph{$s$-separation capacity} of $\Phi$. If $s = M$, we write $\sepcap{\Phi} \coloneqq \sepcaps{M}{\Phi}$ and call $\sepcap{\Phi}$ the \emph{separation capacity} of $\Phi$. 
\end{definition}
The next lemma recasts the $s$-separation capacity as a minimization over linear subspaces that
carry positive pushforward measure.
\begin{lemma}\label{L:sepcap}
    Let $\Phi \colon E \to \mathbb R^{M'}$ be measurable. We have
    \begin{align}\label{eq:sepcap-form1}
        \sepcaps{s}{\Phi} = 2\cdot\min \left\{\dimR{V} \colon V \subseteq \mathbb R^{M'} \text{ linear subspace, } \left(\pushfwd{\restrmeasure{\mathcal{H}^s}{E}}{\Phi}\right)\!(V) >0\right\}.
    \end{align}
\end{lemma}
\begin{proof}
    As shown in \cite{haberle2026function}, it holds that 
    \begin{align}\label{eq:sepcap-form2}
        \sepcaps{s}{\Phi} = 2 \min_{\substack{A \subseteq E \\ \mathcal{H}^s(A)>0}} \dimR{\spanR{\Phi(A)}}.
    \end{align}
    To establish the claim of this lemma, we first prove that \cref{eq:sepcap-form1} holds with ``$\geq$''. To this end, let $U= \spanRinline{\Phi(A)}$, where $A \subseteq E$ with $\mathcal{H}^s(A)>0$ is such that $\sepcaps{s}{\Phi} = 2\cdot\dimRinline{U}$. We then have $(\Phi^{-1}(U) \cap E) \supseteq A$, which implies $(\pushfwd{\restrmeasure{\mathcal{H}^s}{E}}{\Phi})(U) >0$, and the inequality ``$\geq$'' follows.  
    
    \sloppy Turning to the reverse inequality, i.e., ``$\leq$'', let $V\subseteq \mathbb R^{M'}$ be a linear subspace with $\mathcal{H}^s(\Phi^{-1}(V) \cap E)>0$ such that $2\cdot\dimRinline{V}$ equals the right-hand side (RHS) of \cref{eq:sepcap-form1}. Setting $A' \coloneqq \Phi^{-1}(V) \cap E$, we have $\spanRinline{\Phi(A')} \subseteq V$, as $V$ is a linear subspace. Consequently, as $\mathcal{H}^s(A')>0$, \cref{eq:sepcap-form2} yields $\sepcaps{s}{\Phi} \leq 2\cdot  \dimRinline{\spanRinline{\Phi(A')}} \leq 2\cdot \dimRinline{V}$. But $2\cdot\dimRinline{V}$ equals the RHS of \cref{eq:sepcap-form1}. This establishes the inequality ``$\leq$'', and the proof is complete.   
\end{proof}
\cref{L:sepcap} links the $s$-separation capacity to the lower Hausdorff dimension of a measure. Writing $\dim_{\mathrm{H}}(B)$ for the Hausdorff dimension of a set $B \subseteq \mathbb R^{M'}$, the lower Hausdorff dimension of a measure $\mu$ on $\mathbb R^{M'}$ is defined to be 
\begin{align*}
    \underline{\dim}_{\mathrm{H}}(\mu) = \inf \left\{\dim_{\mathrm{H}}(B)\colon B \subseteq \mathbb R^{M'} \text{ measurable, } \mu(B)>0\right\}.
\end{align*}

\begin{lemma}[$s$-separation capacity and lower Hausdorff dimension] For every measurable $\Phi \colon E \to \mathbb R^{M'}$, it holds that
        \begin{align*}
        \sepcaps{s}{\Phi} \geq 2 \cdot \underline{\dim}_{\mathrm{H}} \left(\pushfwd{\restrmeasure{\mathcal{H}^s}{E}}{\Phi}\right).
    \end{align*}
\end{lemma}
\begin{proof}
    The claim follows immediately by relaxing the minimization in \cref{eq:sepcap-form1} from linear subspaces $V \subseteq \mathbb{R}^{M'}$ to arbitrary measurable subsets $B \subseteq \mathbb{R}^{M'}$, upon noting that for every linear subspace $V$, one has $\dim_{\mathrm H}(V)=\dim_{\mathbb R}(V)$.
\end{proof}
Next, we will leverage \cref{L:sepcap} to express the $s$-separation capacity in terms of the geometry of the input set $E$, when $E$ is rectifiable. Recall the notion of a rectifiable set.   
\begin{definition}[Countably $\mathcal{H}^s$-rectifiable set, \cite{ambrosio2000currents,federer2014geometric,simon2014introduction}]
    An $\mathcal{H}^s$-measurable set $E \subseteq \mathbb R^M$ is said to be \emph{countably $\mathcal H^s$-rectifiable} if there is a countable family of Lipschitz maps $\{\psi_k \colon \mathbb R^s \to \mathbb R^M\}_{k \in \mathbb N}$ such that
            \begin{align*}
                \mathcal H^s \left(E \setminus \bigcup_{k \in \mathbb N} \psi_k(\mathbb R^s) \right) =0.
            \end{align*}
\end{definition}
\begin{lemma}[Bi-Lipschitz parametrization, \cite{ambrosio2000currents,federer2014geometric}]\label{L:bi-Lip}
    Let $E \subseteq \mathbb R^M$ be countably $\mathcal H^s$-rectifiable, and let $\lipbound >1$. Then there exist finitely or countably many compact sets $K_j \subset \mathbb R^s$ and $\lipbound$-bi-Lipschitz maps $\psi_j \colon K_j \to \psi_j(K_j) \subseteq E$, indexed by $\mathcal{J}$, such that $\{\psi_j(K_j)\}_{j \in \mathcal{J}}$ are pairwise disjoint and 
    \begin{align*}
        \mathcal H^s \left(E \setminus \bigcup_{j \in \mathcal{J}} \psi_j(K_j) \right) =0.
    \end{align*}
\end{lemma}
Rectifiable sets admit linear approximation properties, formalized through the notion of approximate tangent spaces. For $f \in \mathbb R^M$ and $\lambda>0$, let $\eta_{f,\lambda}\colon \mathbb R^M \to \mathbb R^M, h \mapsto (h-f)/\lambda$. We use $C_c(\mathbb R^M)$ to denote the space of continuous, compactly supported functions on $\mathbb{R}^M$.
\begin{definition}[Approximate tangent space, \cite{simon2014introduction}]
    Suppose that $E \subseteq \mathbb R^M$ is $\mathcal{H}^s$-measurable with $\mathcal{H}^s(E)<\infty$, and let $f \in \mathbb R^M$. We call an $s$-dimensional linear subspace $L \subseteq \mathbb R^M$ the \emph{approximate tangent space of $E$ at $f$} if 
    \begin{align*}
        \lim_{\lambda \to 0^+} \int_{\eta_{f,\lambda}(E)} \phi\,\mathrm{d}\mathcal{H}^s = \int_L \phi\,\mathrm{d}\mathcal{H}^s, \quad \text{for all $\phi \in C_c(\mathbb R^M)$}.
    \end{align*}
    We shall write $T_fE \coloneqq L$.
\end{definition}
If $E \subseteq \mathbb{R}^M$ is countably $\mathcal{H}^s$-rectifiable, then the approximate tangent space $T_fE$ exists for $\mathcal{H}^s$-a.e. $f \in E$ \cite{simon2014introduction,federer2014geometric}. This allows us to differentiate Lipschitz maps on $E$. For a Lipschitz map $\Phi \colon E \to \mathbb{R}^{M'}$ and a Lipschitz extension $\overline{\Phi} \colon \mathbb{R}^M \to \mathbb{R}^{M'}$ of $\Phi$, we write $\mathrm{d}\Phi_f \coloneqq \mathrm{d}\overline{\Phi}_f$ for the differential of $\overline{\Phi}$ at $f$. By \cite[Theorem~11.4]{maggi2012sets}, the restriction $\mathrm{d}\Phi_f\rvert_{T_fE} \colon T_fE \to \mathbb{R}^{M'}$ exists for $\mathcal{H}^s$-a.e. $f \in E$ and, by \cite[Lemma~11.5]{maggi2012sets}, it is independent of the choice of the extension $\overline{\Phi}$.  
With this notion at hand, we establish a connection between the separation capacity and the tangential geometry of $E$.
\begin{lemma}[$s$-separation capacity and approximate tangent spaces]\label{L:lower-bound}
Let $E$ be countably $\mathcal{H}^s$-rectifiable with $\mathcal{H}^s(E)<\infty$ and  $\Phi\colon E\to \mathbb R^{M'}$ be Lipschitz.
Then, 
    \begin{align}\label{eq:lower-bound-tangent}
        \sepcaps{s}{\Phi} \geq 2 \cdot \essinf{f \in E}{\rank{\restr{\mathrm{d}\Phi_f}{T_fE}}}.
    \end{align}
\end{lemma}
\begin{proof}
    To establish the claim, it suffices to show that for every linear subspace $V\subseteq \mathbb R^{M'}$, 
    \begin{align*}
        \left(\mathcal{H}^s(\Phi^{-1}(V) \cap E) >0\right) \implies  \left(\essinf{f \in E}{\rank{\restr{\mathrm{d}\Phi_f}{T_fE}}} \leq \dimR{V}\right).
    \end{align*}
    Fix a linear subspace $V\subseteq \mathbb R^{M'}$ with
    $\mathcal{H}^s(\Phi^{-1}(V) \cap E)>0$ and set $A \coloneqq \Phi^{-1}(V) \cap E$. 
    By the locality of approximate tangent spaces
    \cite[Proposition 10.5]{maggi2012sets}, $T_fA = T_fE$ for $\mathcal{H}^s$-a.e.
    $f \in A$. Since $\restr{\Phi}{A}\colon A \to V$, we thus have, for
    $\mathcal{H}^s$-a.e. $f \in A$,
    \begin{align*}
        \mathrm{d}\Phi_f(T_fE) = \mathrm{d}\Phi_f(T_fA) \subseteq T_{\Phi(f)}V=V.
    \end{align*}
    But then $\dimRinline{\mathrm{d}\Phi_f(T_fE)} \leq \dimRinline{V}$, for $\mathcal{H}^s$-a.e. $f \in A$, and consequently, it holds that
    \begin{align*}
        \essinf{f \in E}{\dimR{\mathrm{d}\Phi_f(T_fE)}} \leq \dimR{V}.
    \end{align*}
    Upon noting that $\rankinline{\restrinline{\mathrm{d}\Phi_f}{T_fE}} = \dimRinline{\mathrm{d}\Phi_f(T_fE)}$, this completes the proof. 
\end{proof}
\begin{remark}
    The inequality in \cref{eq:lower-bound-tangent} is tight. Indeed, if $\Phi$ is linear and $E$ an $s$-dimensional linear subspace of $\mathbb R^M$, then \cref{eq:lower-bound-tangent} holds with equality, see \cite{haberle2026function}. 
\end{remark}
The subspace characterization of the $s$-separation capacity presented in \cref{L:sepcap} also gives an upper bound in terms of a second-moment matrix. 
\begin{lemma}[Upper bound]\label{L:upper-bound}
    Let $\Phi \colon E \to \mathbb R^{M'}$ be measurable such that $\int_{E} \|\Phi(f)\|^2\,\mathrm{d}\mathcal{H}^s(f) < \infty$. Then, 
    \begin{align}\label{eq:upper-bound}
        \sepcaps{s}{\Phi} \leq 2\cdot\rank{\int_{\mathbb R^{M'}} y y^\mathsf{T}\,\mathrm{d}\pushfwd{\restrmeasure{\mathcal{H}^s}{E}}{\Phi}}.
    \end{align}
\end{lemma}
\begin{proof}
    \sloppy First note that the RHS of \cref{eq:upper-bound} is well-defined as $\int_{\mathbb R^{M'}} y y^\mathsf{T}\,\mathrm{d}\pushfwdinline{\restrmeasure{\mathcal{H}^s}{E}}{\Phi}= \int_{E} \Phi(f) \Phi(f)^\mathsf{T}\,\mathrm{d}\mathcal{H}^s(f)$ and by the Cauchy--Schwarz inequality, we have, for $k,\ell \in \{1,\ldots,M'\}$,
    \begin{align*}
        \int_E \lvert (\Phi(f))_k (\Phi(f))_\ell\rvert\,\mathrm{d}\mathcal{H}^s(f) &\leq \left(\int_E \lvert (\Phi(f))_k\rvert^2\,\mathrm{d}\mathcal{H}^s(f)\right)^{1/2}\left(\int_E \lvert (\Phi(f))_\ell\rvert^2\,\mathrm{d}\mathcal{H}^s(f)\right)^{1/2} \\
        &\leq \int_{E} \|\Phi(f)\|^2\,\mathrm{d}\mathcal{H}^s(f) < \infty. 
    \end{align*}
    We write $C_\Phi \coloneqq \int_{\mathbb R^{M'}} y y^\mathsf{T}\,\mathrm{d}\pushfwd{\restrmeasure{\mathcal{H}^s}{E}}{\Phi}$ and note that for every $w \in \ker{C_\Phi}$, 
    \begin{align*}
        0 = w^\mathsf{T}C_\Phi w = \int_{\mathbb R^{M'}} \left(\innerprod{y}{w}\right)^2\,\mathrm{d}\pushfwd{\restrmeasure{\mathcal{H}^s}{E}}{\Phi},
    \end{align*}
    which implies
    \begin{align}\label{eq:inner-prod-zero-ae}
        \innerprod{y}{w} = 0, \quad \text{for $\pushfwd{\restrmeasure{\mathcal{H}^s}{E}}{\Phi}$-a.e. $y \in \mathbb R^{M'}$}.
    \end{align}
    Let $V= (\ker{C_\Phi})^\perp$. To establish the claim, it suffices to show that
    \begin{align}\label{eq:supp-V}
        \supp{\pushfwd{\restrmeasure{\mathcal{H}^s}{E}}{\Phi}} \subseteq V.
    \end{align}
    Indeed, \cref{eq:supp-V} implies that $(\pushfwd{\restrmeasure{\mathcal{H}^s}{E}}{\Phi})(V)>0$, as $(\pushfwd{\restrmeasure{\mathcal{H}^s}{E}}{\Phi})(\mathbb R^{M'}) = \mathcal{H}^s(E)>0$, and thus $\sepcaps{s}{\Phi} \leq 2\cdot\dimR{V} = 2\cdot\rank{C_\Phi}$. 
    To see that \cref{eq:supp-V} holds, let $y_0\in \suppinline{\pushfwd{\restrmeasure{\mathcal{H}^s}{E}}{\Phi}}$, and suppose for the sake of contradiction that $y_0 \notin V$, i.e., $\innerprod{y_0}{w} \neq 0$ for some $w \in \ker{C_\Phi}$. But then there is an open neighborhood $N_{y_0}$ containing $y_0$ such that 
    \begin{align*}
        \innerprod{y}{w} \neq 0, \quad \text{for all $y \in N_{y_0}$}.  
    \end{align*}
    As $y_0 \in \suppinline{\pushfwdinline{\restrmeasure{\mathcal{H}^s}{E}}{\Phi}}$, we have $(\pushfwd{\restrmeasure{\mathcal{H}^s}{E}}{\Phi})(N_{y_0})>0$. This is in contradiction to \cref{eq:inner-prod-zero-ae}, thereby completing the proof. 
\end{proof}
Application of the area formula \cite{federer2014geometric} to the RHS of \cref{eq:upper-bound} yields
\begin{align*}
    \sepcaps{s}{\Phi} &\leq 2\cdot\rank{\int_{\mathbb R^{M'}} y y^\mathsf{T}\,\mathrm{d}\pushfwd{\restrmeasure{\mathcal{H}^s}{E}}{\Phi}} \\
    &= 2\cdot\rank{\int_{E} \Phi(f) \Phi(f)^\mathsf{T}\,\mathrm{d}\mathcal{H}^s(f)} \\
    &= 2\cdot\rank{\sum_{j \in \mathcal{J}}\int_{K_j} \Phi(\psi_j(x)) \Phi(\psi_j(x))^{\mathsf{T}} \mathbf{J}(\mathrm{d}(\psi_j)_x)\,\mathrm{d}x},
\end{align*}
where $\mathbf{J}(\mathrm{d}(\psi_j)_x)$ denotes the Jacobian of the linear map $\mathrm{d}(\psi_j)_x$.
\begin{corollary}\label{cor:real-analytic}
    Let $E \subseteq \mathbb R^M$ be countably $\mathcal{H}^s$-rectifiable with bi-Lipschitz parametrization $\{\psi_j \colon K_j \to \psi_j(K_j)\}_{j \in \mathcal{J}}$, and let $\Phi \colon E \to \mathbb R^{M'}$. If, for each $j \in \mathcal{J}$, the map $\Phi \circ \psi_j$
    admits a real-analytic extension to a connected open neighborhood of $K_j$, then
    \begin{align}\label{eq:2ndmoment-real-analytic}
        \sepcaps{s}{\Phi} = 2\min_{j \in \mathcal{J}} \rank{\int_{K_j} \Phi(\psi_j(x)) \Phi(\psi_j(x))^{\mathsf{T}} \mathbf{J}(\mathrm{d}(\psi_j)_x)\,\mathrm{d}x}.
    \end{align}
\end{corollary}
\begin{proof}
    Since $\Phi \circ \psi_j$ is in particular continuous and $K_j$ is compact, each integral in \cref{eq:2ndmoment-real-analytic} is finite.
    As established in \cite{haberle2026function}, $\sepcaps{s}{\Phi} = \min_{j \in \mathcal{J}} \sepcapsinline{s}{\restrinline{\Phi}{\psi_j(K_j)}}$. Application of \cref{L:upper-bound} to $\restrinline{\Phi}{\psi_j(K_j)}$ yields
    \begin{align}\label{eq:real-analytic1}
        \sepcaps{s}{\restr{\Phi}{\psi_j(K_j)}} \leq 2\cdot\rank{\underbrace{\int_{\mathbb R^{M'}} y y^\mathsf{T}\,\mathrm{d}\pushfwd{\restrmeasure{\mathcal{H}^s}{\psi_j(K_j)}}{\Phi}}_{\eqqcolon C_{\Phi_j}}}. 
    \end{align}
    To show that \cref{eq:real-analytic1} holds with equality, suppose for the sake of contradiction that there is a linear subspace $U \subseteq \mathbb R^{M'}$ with $\dimRinline{U} < \rankinline{C_{\Phi_j}}$ and $(\pushfwdinline{\restrmeasure{\mathcal{H}^s}{\psi_j(K_j)}}{\Phi})(U)>0$. Further, let $V = (\ker{C_{\Phi_j}})^\perp$, so that $\dimRinline{V} = \rankinline{C_{\Phi_j}}$ and $(\pushfwdinline{\restrmeasure{\mathcal{H}^s}{\psi_j(K_j)}}{\Phi})(V)>0$, as argued in the proof of \cref{L:upper-bound}. Since $\dimRinline{U^\perp} > M'- \rankinline{C_{\Phi_j}}$ and $\dimRinline{V} = \rankinline{C_{\Phi_j}}$,
    \begin{align*}
        \dimR{U^\perp \cap V} &= \dimR{U^\perp}+\dimR{V}-\dimR{U^\perp + V} \\
        &> M'- \rankinline{C_{\Phi_j}} + \rankinline{C_{\Phi_j}} -\dimR{U^\perp + V} \\
        &= M'- \dimR{U^\perp + V} \\
        &\geq 0.
    \end{align*}
    We next note that $\Phi(\psi_j(K_j))^\perp \subseteq \ker{C_{\Phi_j}} = V^\perp$, so that $V \subseteq \spanRinline{\Phi(\psi_j(K_j))}$.
    One can thus choose $v \in V \cap U^\perp$ such that $v\neq 0$ and such that there is an $x_v \in K_j$ satisfying $\innerprod{\Phi(\psi_j(x_v))}{v} \neq 0$. Setting $\varphi(x) \coloneqq \innerprod{(\Phi \circ \psi_j)(x)}{v}$, $x \in K_j$, then, by assumption, $\varphi$ extends real-analytically to a connected open
    neighborhood $\Omega_j$ of $K_j$. Moreover, $\varphi$ is not identically zero, as $\varphi(x_v) \neq 0$. However, by construction, $\varphi(x) = 0$, for all $x \in (\Phi \circ \psi_j)^{-1}(U)$. This establishes the contradiction, as a nontrivial real-analytic function on the connected open set $\Omega_j$ cannot
    vanish \cite{mityagin2020zerorealanalytic} on a set of measure   
    \begin{align*}
        \mathcal{L}^s\!\left((\Phi \circ \psi_j)^{-1}(U)\right)  &\geq \mathcal{H}^s\!\left(\psi_j\left((\Phi \circ \psi_j)^{-1}(U)\right)\right) (\lip{\psi_j})^{-s} \\ 
        &= \mathcal{H}^s\!\left(\Phi^{-1}(U) \cap \psi_j(K_j)\right) \\
        &= \left(\pushfwd{\restrmeasure{\mathcal{H}^s}{\psi_j(K_j)}}{\Phi}\right)(U)\\
        &>0. 
    \end{align*}
    Therefore, \cref{eq:real-analytic1} holds with equality. Application of the area formula yields the desired expression, and the proof is complete. 
\end{proof}
\begin{corollary}\label{cor:lower-bound-real-analytic}
    Under the hypotheses of \cref{cor:real-analytic}, 
    \begin{align*}
        \sepcaps{s}{\Phi} \geq \inf_{j \in \mathcal{J}} \frac{2\left(\int_{K_j} \lVert \Phi(\psi_j(x))\rVert^2\mathbf{J}(\mathrm{d}(\psi_j)_x)\,\mathrm{d}x\right)^2}{\int_{K_j}\int_{K_j}  \lvert\langle \Phi(\psi_j(x)), \Phi(\psi_j(y))\rangle\rvert^2\mathbf{J}(\mathrm{d}(\psi_j)_x)\mathbf{J}(\mathrm{d}(\psi_j)_y)\,\mathrm{d}x\mathrm{d}y}.
    \end{align*}
\end{corollary}
\begin{proof}
    The matrix $C_{\Phi_j} \coloneqq \int_{K_j} \Phi(\psi_j(x)) \Phi(\psi_j(x))^{\mathsf{T}} \mathbf{J}(\mathrm{d}(\psi_j)_x)\,\mathrm{d}x$ is symmetric, so that $r_j\coloneqq\rank{C_{\Phi_j}}$ is given by the number of nonzero eigenvalues of $C_{\Phi_j}$. Denoting by $\{\lambda _k\}_{k=1}^{r_j}$ the nonzero eigenvalues of $C_{\Phi_j}$, we have 
    \begin{align*}
        \trace{C_{\Phi_j}} = \sum_{k=1}^{r_j} \lambda_k\quad \text{and}\quad 
        \trace{C_{\Phi_j}^2} = \sum_{k=1}^{r_j} \lambda_k^2.
    \end{align*}
    Application of the Cauchy--Schwarz inequality yields
    \begin{align*}
        \trace{C_{\Phi_j}}^2 = \left(\sum_{k=1}^{r_j} \lambda_k\right)^2 \leq r_j \sum_{k=1}^{r_j} \lambda_k^2 = r_j \trace{C_{\Phi_j}^2},
    \end{align*}
    so that
    \begin{align*}
        \rank{C_{\Phi_j}} \geq \frac{\trace{C_{\Phi_j}}^2}{\trace{C_{\Phi_j}^2}}.
    \end{align*}
    Computing 
    \begin{align*}
        \trace{C_{\Phi_j}} = \int_{K_j} \trace{\Phi(\psi_j(x)) \Phi(\psi_j(x))^{\mathsf{T}}} \mathbf{J}(\mathrm{d}(\psi_j)_x)\,\mathrm{d}x = \int_{K_j} \lVert \Phi(\psi_j(x))\rVert^2\mathbf{J}(\mathrm{d}(\psi_j)_x)\,\mathrm{d}x  
    \end{align*}
    and 
    \begin{align*}
        \trace{C_{\Phi_j}^2} &= \trace{\int_{K_j}\int_{K_j} \Phi(\psi_j(x)) \Phi(\psi_j(x))^{\mathsf{T}} \Phi(\psi_j(y)) \Phi(\psi_j(y))^{\mathsf{T}} \mathbf{J}(\mathrm{d}(\psi_j)_x)\mathbf{J}(\mathrm{d}(\psi_j)_y)\,\mathrm{d}x\mathrm{d}y} \\
        &= \int_{K_j} \int_{K_j} \lvert\langle \Phi(\psi_j(x)), \Phi(\psi_j(y))\rangle\rvert^2\mathbf{J}(\mathrm{d}(\psi_j)_x)\mathbf{J}(\mathrm{d}(\psi_j)_y)\,\mathrm{d}x\mathrm{d}y
    \end{align*}
    gives the desired bound upon application of \cref{cor:real-analytic}.
\end{proof}
\section{Application: Separation capacity of scattering networks}\label{sec:sc-snet}
\subsection{Scattering network theory}
We now instantiate the preceding capacity bounds for finite-dimensional scattering networks. Input signals are functions $f \colon \mathbb Z/M\mathbb Z \to \mathbb C$. 
Scattering networks à la Wiatowski and B\"olcskei \cite{wiatowski2017deep} are built from sequences of the form $\{(\Psi_n,\rho_n,P_n)\}_{n \in \mathbb N}$, where the $n$th network layer, $n \in \mathbb N$, is determined by the triplet $(\Psi_n,\rho_n,P_n)$ consisting of (i) a frame $\Psi_n$ generated by a countable set of functions $\{\chi_n\} \cup \{g_{\lambda_n}\}_{\lambda_n \in \Lambda_n} \subseteq \mathbb C^{\mathbb Z/M \mathbb Z}$ such that
\begin{align}\label{eq:frame-condition}
    A_n \lVert f \rVert^2 \leq \lVert f \ast \chi_n \rVert^2 +\sum_{\lambda_n \in \Lambda_n} \lVert f \ast g_{\lambda_n} \rVert^2 \leq B_n \lVert f \rVert^2, \quad \text{for all $f \in \mathbb C^{\mathbb Z/M \mathbb Z}$},
\end{align}
with constants $0<A_n\leq B_n<\infty$, (ii) a pointwise nonlinearity $\rho_n \colon \mathbb C \to \mathbb C$, and (iii) a pooling operator $P_n \colon \mathbb C^{\mathbb Z/M \mathbb Z} \to \mathbb C^{\mathbb Z/M \mathbb Z}$.
The operation in the node $\lambda_n \in \Lambda_n$ in the $n$th network layer is set to be 
\begin{align}\label{eq:operation-node}
    U[\lambda_n]\colon f \mapsto P_n(\rho_n(f \ast g_{\lambda_n})), \quad f \in \mathbb C^{\mathbb Z/M \mathbb Z},
\end{align}
where $\left(\rho_n(f \ast g_{\lambda_n})\right)(k)\coloneqq \rho_n((f \ast g_{\lambda_n})(k))$, $k \in \mathbb Z/M\mathbb Z$. Extending \cref{eq:operation-node} according to
\begin{align*}
    U[(\lambda_1,\ldots,\lambda_n)]f = U[\lambda_n]\cdots U[\lambda_1]f, \quad f \in \mathbb C^{\mathbb Z/M\mathbb Z},
\end{align*}
for $(\lambda_1,\ldots,\lambda_n) \in \Lambda_1 \times \cdots\times \Lambda_n \eqqcolon \Lambda_1^n$, and setting $\Lambda_1^0 \coloneqq \{e\}$ as well as $U[e]f \coloneqq f$, for all $f \in \mathbb C^{\mathbb Z/M\mathbb Z}$, 
the \emph{scattering network of depth $n_{\mathrm{d}} \in \mathbb N$} is given by 
\begin{align*}
    \Phi \colon \mathbb C^{\mathbb Z/M \mathbb Z} \to \left(\mathbb C^{\mathbb Z/M \mathbb Z}\right)^{\bigcup_{n=0}^{n_{\mathrm{d}}} \Lambda_1^n}, \quad f \mapsto \bigcup_{n=0}^{n_{\mathrm{d}}} \Phi^n(f). 
\end{align*}
Here, $\Phi^n(f) \coloneqq \{(U[q]f) \ast \chi_{n+1}\}_{q \in \Lambda_1^n}$ denotes the output of the $n$th network layer.

\paragraph{Computing the separation capacity of scattering networks} 
To analyze the separation capacity, we shall identify $\mathbb{Z}/M\mathbb{Z}$ with
$\{0,\ldots,M-1\}$ equipped with addition modulo $M$, and employ the usual chain of identifications $\mathbb C^{\mathbb Z/M\mathbb Z} \simeq \mathbb C^M \simeq \mathbb R^{2M}$, so that we are in the setting of the previous section. To see what these identifications mean in the context of binary classification based on the feature extractor $\Phi \colon E \subseteq \mathbb C^M \to \mathbb C^{M'}$, we note that for $f \in E$ and $w \in \mathbb C^{M'}$,
\begin{align}\label{eq:inner-prod-real-complex}
    \Re(\innerprod{\Phi(f)}{w}) = \innerprod{\Re(\Phi(f))}{\Re(w)} + \innerprod{\Im(\Phi(f))}{\Im(w)} = \innerprod{\begin{pmatrix}
        \Re(\Phi(f)) \\ \Im(\Phi(f))
    \end{pmatrix}}{\begin{pmatrix}
        \Re(w) \\ \Im(w)
    \end{pmatrix}}.
\end{align}
This motivates the following notion of $\Phi$-separability for complex-valued maps $\Phi \colon E \subseteq \mathbb C^M \to \mathbb C^{M'}$. A dichotomy $\{F_+,F_-\}$ of an $N$-point set $F \subseteq E$ is $\Phi$-separable if there is a $w \in \mathbb C^{M'}$ such that\footnote{Equivalently, one may consider the sign of $\Im(\innerprod{\Phi(f)}{w})$.} 
\begin{align*}
    \Re(\innerprod{\Phi(f)}{w}) &>0, \quad \text{if $f \in F_+$},\\
    \Re(\innerprod{\Phi(f)}{w}) &<0, \quad \text{if $f \in F_-$}.
\end{align*}
Furthermore, we define the $s$-separation capacity of $\Phi$ as the $s$-separation capacity of the map 
\begin{align*}
    \widetilde{\Phi}\colon \widetilde{E} \subseteq \mathbb R^{2M} \to \mathbb R^{2M'}, \begin{pmatrix}
        f' \\ f''
    \end{pmatrix} \mapsto\begin{pmatrix}
        \Re(\Phi(f'+if'')) \\
        \Im(\Phi(f'+if''))
    \end{pmatrix}, 
\end{align*} 
where $\widetilde{E} \coloneqq \iota(E)$ and $\iota \colon \mathbb{C}^M \to \mathbb{R}^{2M}$, $f \mapsto (\Re(f)^\mathsf{T},\Im(f)^\mathsf{T})^\mathsf{T}$, denotes the decomplexification map. We shall write $\sepcapsinline{s}{\Phi} \coloneqq \sepcapsinline{s}{\widetilde{\Phi}}$.

We henceforth restrict ourselves to networks built from $\{(\Psi,\rho,\mathrm{Id})\}_{n \in \mathbb N}$, where $\Psi$ is the frame generated by some functions $\{\chi\} \cup \{g_{\lambda}\}_{\lambda \in \Lambda} \subseteq \mathbb C^{\mathbb Z/M \mathbb Z}$ in the sense of \cref{eq:frame-condition}, and $\rho(z)=z^d$, $z \in \mathbb C$, is a fixed monomial. In this setting, $\Lambda_1^n = \Lambda^n$, the $n$-fold Cartesian product of $\Lambda$, and $\Lambda_1^0 = \Lambda^0 \coloneqq \{e\}$.
Given the input dataset $E \subseteq \mathbb C^{M} \simeq\mathbb C^{\mathbb Z/M\mathbb Z}$ and the parameter $s \geq 0$, the design problem is to choose the frame $\Psi$ so as to maximize $\sepcaps{s}{\Phi}$. We treat two model geometries: sparse signals and polynomially parametrized rectifiable sets.

\subsection{Sparse signals}
We first take $E$ to be a sparse-signal model. Let $\Xi = \{\xi_k\}_{k \in \mathcal{K}}$ be a frame for $\mathbb C^M$, and suppose that every subset of cardinality $s$ of $\Xi$ is linearly independent, so that for each $S \subseteq \mathcal{K}$ with $\lvert S\rvert =s$, $E_S\coloneqq \spanCinline{\{\xi_k\}_{k \in S}}$ is of dimension $s$. The set of $s$-sparse signals is given as $E \coloneqq \bigcup_{S \subseteq \mathcal{K}, \lvert S\rvert =s} E_S$.
After decomplexification according to $\widetilde{E} = \iota(E) \subset \mathbb R^{2M}$, this model becomes a union of $2s$-dimensional $\mathbb R$-linear subspaces: 
\begin{align}
    \widetilde{E} = \bigcup_{S \subseteq \mathcal{K}, \lvert S\rvert =s} \spanR{\left\{\begin{pmatrix}
        \Re(\xi_k) \\ \Im(\xi_k)
    \end{pmatrix},\begin{pmatrix}
        -\Im(\xi_k) \\ \Re(\xi_k) 
    \end{pmatrix}\right\}_{k \in S}}.
\end{align}
Indeed, if $\{c_k',c_k''\}_{k \in S} \subset \mathbb R$ satisfy
\begin{align*}
    \sum_{k \in S} \left(c_k' \begin{pmatrix}
        \Re(\xi_k) \\ \Im(\xi_k)
    \end{pmatrix}+ c_k ''\begin{pmatrix}
        -\Im(\xi_k) \\ \Re(\xi_k) 
    \end{pmatrix} \right) =0, 
\end{align*}
then 
\begin{align*}
    \sum_{k \in S} (c_k'+ic_k'')\xi_k =0. 
\end{align*}
The $\mathbb C$-linear independence of $\{\xi_k\}_{k\in S}$ yields $c_k'=c_k''=0$, for all $k \in S$, and hence the set 
\begin{align}\label{eq:dimR-2s}
    \left\{\begin{pmatrix}
        \Re(\xi_k) \\ \Im(\xi_k)
    \end{pmatrix},\begin{pmatrix}
        -\Im(\xi_k) \\ \Re(\xi_k) 
    \end{pmatrix}\right\}_{k \in S} \quad \text{ is linearly independent.}
\end{align}
In particular, the set of $s$-sparse signals in $\mathbb C^M$ is countably $\mathcal{H}^{2s}$-rectifiable.
\paragraph{Upper bound on $\sepcaps{2s}{\Phi}$} The upper bound starts from an exact expression for $\sepcaps{2s}{\Phi}$ in terms of complex spans. 
\begin{proposition}\label{prop:sc-complex}
    Let $\Phi \colon E \to \mathbb C^{M'}$ be real-analytic. Then,
    \begin{align*}
        \sepcaps{2s}{\Phi} = 2\min_{\substack{S \subseteq \mathcal{K} \\ \lvert S \rvert =s}} \dimC{\spanC{\left\{\begin{pmatrix}
        \varphi^\mathsf{T},\overline{\varphi}^\mathsf{T}
    \end{pmatrix}^\mathsf{T}\colon \varphi \in \Phi(E_S)\right\}}}. 
    \end{align*}
\end{proposition}
\begin{proof}
Thanks to \cref{eq:dimR-2s}, we can employ \cite{haberle2026function} to get
\begin{align}\label{eq:sc-sparse-1}
    \sepcaps{2s}{\Phi} &= \sepcaps{2s}{\widetilde\Phi}
    = \min_{\substack{S \subseteq \mathcal{K} \\ \lvert S \rvert =s}} \sepcap{\widetilde\Phi \circ \widetilde\sigma_S},
\end{align}
where
\begin{align*}
    \widetilde\sigma_S \colon \mathbb R^{2s}\to \mathbb R^{2M}, \begin{pmatrix}
        c' \\ c''
    \end{pmatrix} \mapsto \sum_{j=1}^s \left(c_j' \begin{pmatrix}
        \Re(\xi_{k_{S,j}}) \\ \Im(\xi_{k_{S,j}})
    \end{pmatrix}+ c_j'' \begin{pmatrix}
        -\Im(\xi_{k_{S,j}}) \\ \Re(\xi_{k_{S,j}}) 
    \end{pmatrix} \right)
\end{align*}
with the labeling $S=\{{k_{S,j}}\}_{j=1}^s$. 
Further, as $\widetilde{\Phi}\circ \widetilde{\sigma}_S$ is real-analytic, we have, by \cite{haberle2026function},
\begin{align}\label{eq:sc-sparse-2}
    \sepcap{\widetilde\Phi \circ \widetilde\sigma_S} &= 2\cdot \dimR{\spanR{(\widetilde{\Phi}\circ \widetilde{\sigma}_S)(\mathbb R^{2s})}}.
\end{align}
Let $\sigma_S \colon \mathbb C^s \to \mathbb C^M, c \mapsto \sum_{j=1}^s c_j \xi_{k_{S,j}}$ and  $T \coloneqq \begin{pmatrix}
    I_{M'} &iI_{M'} \\I_{M'} &-iI_{M'}
\end{pmatrix}$. 
Then, for $\begin{pmatrix}
        c' \\ c''
    \end{pmatrix} \in \mathbb R^{2s}$,  
\begin{align*}
    \begin{pmatrix}
        (\Phi\circ\sigma_S)(c'+ic'') \\
        \overline{(\Phi\circ\sigma_S)(c'+ic'')}
    \end{pmatrix} = T (\widetilde{\Phi}\circ\widetilde{\sigma}_S)(c'+ic'').
\end{align*}
As $T$ is invertible, we have
\begin{align} \label{eq:sc-sparse-3}
    \dimR{\spanR{(\widetilde{\Phi}\circ \widetilde{\sigma}_S)(\mathbb R^{2s})}} = \dimC{\spanC{\left\{\begin{pmatrix}
        (\Phi\circ \sigma_S)(c) \\ \overline{(\Phi\circ \sigma_S)(c)}
    \end{pmatrix}\colon c \in \mathbb C^s\right\}}}.
\end{align}
Combining \cref{eq:sc-sparse-1,eq:sc-sparse-2,eq:sc-sparse-3} completes the proof.
\end{proof}
Thus, maximizing $\sepcaps{2s}{\Phi}$ requires  
\begin{align*}
    \dimC{\spanC{\left\{\begin{pmatrix}
        (\Phi\circ \sigma_S)(c) \\ \overline{(\Phi\circ \sigma_S)(c)}
    \end{pmatrix}\colon c \in \mathbb C^s\right\}}} \text{ to be maximized for all $S\subseteq \mathcal{K}$ with $\lvert S \rvert=s$.}
\end{align*}

To derive an upper bound on $\sepcaps{2s}{\Phi}$, we establish a one-node estimate. Consider the map $f \mapsto \rho(f\ast g_\lambda)$, $f \in E_S$, and set $H_{\lambda,S} \coloneqq (\bigcup_{k \in S} \suppinline{\widehat{\xi_k}}) \cap \suppinline{\widehat{g_\lambda}}$. 
The $d$-fold sumset of $H_{\lambda,S} \subseteq \mathbb Z/M\mathbb Z$, formed with
respect to the group operation $+$ on $\mathbb{Z}/M\mathbb{Z}$, is denoted by
\begin{align*}
    H_{d,\lambda,S} \coloneqq \langle H_{\lambda,S} \rangle_d \coloneqq
    \underbrace{H_{\lambda,S} + \cdots + H_{\lambda,S}}_{d \text{ times}}
    = \{h_1 + \cdots + h_d \colon h_1,\ldots,h_d \in H_{\lambda,S}\}.
\end{align*}
\begin{lemma}\label{L:UP_one_node}
    For $S \subseteq \mathcal{K}$ with $\lvert S \rvert =s$, and 
    $\lambda \in \Lambda$, it holds that  
    \begin{align}\label{eq:bound_struct_input_1node}
        \dimC{\spanC{\left\{\rho(f \ast g_\lambda) \colon f \in E_S\right\}}} \leq \min \left\{\binom{s-1+d}{d}, \left\lvert H_{d,\lambda,S}\right\rvert\right\}.
    \end{align}
    Furthermore, we have
    \begin{align}\label{eq:bound_1node_filtered}
        \dimC{\spanC{\left\{\rho(f \ast g_\lambda)\ast \chi \colon f \in E_S\right\}}} \leq \min \left\{\binom{s-1+d}{d}, \left\lvert H_{d,\lambda,S} \cap \supp{\widehat{\chi}}\right\rvert\right\}.
    \end{align}
\end{lemma}
\begin{proof}
    As the DFT matrix $F_M$ is linear and invertible, we may equivalently consider the image of the map $f \mapsto F_M(\rho(f \ast g_\lambda))$ and compute, for $k \in \{0,\dots,M-1\}$,
    \begin{align*}
        \left(F_M\left(\rho(f \ast g_\lambda)\right)\right)_k  &= \frac{1}{M^{d-1}}\sum_{\substack{j_1,\dots,j_d \in \{0,\dots,M-1\} \\ j_1 +\cdots + j_d \equiv k \Mod{M}}} \widehat{f}_{j_1}(\widehat{g_\lambda})_{j_1} \cdots \widehat{f}_{j_d}(\widehat{g_\lambda})_{j_d}\\
        &=\frac{1}{M^{d-1}}\sum_{\substack{\alpha \in \mathbb N_0^M, \lvert \alpha \rvert =d \\ \sum_{t=0}^{M-1} t\alpha_t \equiv k \Mod{M}}} \binom{d}{\alpha} \widehat{g_\lambda}^{\alpha} \widehat{f}^\alpha.
    \end{align*}
    Thus, we can write
    \begin{align*}
        F_M(\rho(f\ast g_\lambda)) =A v_{M,d}(\widehat{f}),
    \end{align*}
    where $A$ is an $M \times \binom{M-1+d}{d}$ matrix with entries
    \begin{align*}
        A_{k\alpha} = \begin{cases}
            M^{1-d} \binom{d}{\alpha} \widehat{g_\lambda}^\alpha, &\text{if $\sum_{t=0}^{M-1}t \alpha_t \equiv k \Mod{M}$ and $\suppinline{\alpha} \subseteq H_{\lambda,S}$,} \\
            0, &\text{otherwise,}
        \end{cases}
    \end{align*}
    for $k \in \{0,\dots,M-1\}$ and $\alpha \in \mathbb N_0^M$ with $\lvert \alpha \rvert =d$ in degree lexicographic order. Here,
    \begin{align*}
        v_{M,d} \colon \mathbb C^M \to \mathbb C^{\binom{M-1+d}{d}}, z \mapsto \left(z^{\alpha}\right)_{\substack{\alpha \in \mathbb N_0^M \\ \lvert \alpha \rvert = d}} 
    \end{align*}
    denotes the vector Veronese map in degree lexicographic order. 
    Note that each column of $A$ has at most one nonzero entry (since for a fixed column $\alpha$, $(\sum_{t=0}^{M-1}t \alpha_t)\bmod{M}$ results in a unique value in $\{0,\ldots,M-1\}$), and consequently, the nonzero rows of $A$ are linearly independent. Thus, the rank of $A$ is given by the number of nonzero rows of $A$; that is,
    \begin{align}
        \rank{A} &= \left\lvert \left\{\left(\sum_{t=0}^{M-1} t \alpha_t\right) \bmod{M} \colon \alpha \in \mathbb N_0^M, \lvert \alpha \rvert =d, \suppinline{\alpha} \subseteq H_{\lambda,S} \right\}\right\rvert  
        = \left\lvert H_{d,\lambda,S}\right\rvert. \label{eq:rank_A_struct_inputs}
    \end{align}
    By \cite[Example 1]{feinberg1981generalization}, we have 
     \begin{align}\label{eq:dim_vd_FS_struct_inputs}
         \dimC{\spanC{v_{M,d}(F_M(E_S))}} = \binom{s-1+d}{d}.
     \end{align}
     From \cref{eq:rank_A_struct_inputs,eq:dim_vd_FS_struct_inputs} the desired upper bound in \cref{eq:bound_struct_input_1node} follows.
     
     To establish \cref{eq:bound_1node_filtered}, note that $F_M(\rho(f \ast g_\lambda) \ast \chi) = \diag{\widehat{\chi}}A v_{M,d}(\widehat{f})$. As $\rank{A}$ is given by the number of rows corresponding to $H_{d,\lambda,S}$, $\rank{\diag{\widehat{\chi}}A} = \lvert H_{d,\lambda,S} \cap \supp{\widehat{\chi}}\rvert$. This, together with \cref{eq:dim_vd_FS_struct_inputs}, proves \cref{eq:bound_1node_filtered}.  
\end{proof}

\begin{theorem} For the scattering network $\Phi$ of depth $n_{\mathrm{d}}$, we have 
\begin{align*}
    \sepcaps{2s}{\Phi} &\leq 4 \min_{\substack{S \subseteq \mathcal{K} \\ \lvert S \rvert =s}}
    \Biggl(\biggl\lvert \Bigl(\bigcup_{k \in S} \supp{\widehat{\xi_k}}\Bigr)
    \cap \supp{\widehat{\chi}} \biggr\rvert \\
    &\qquad + \sum_{\lambda \in \Lambda}
    \min\left\{\binom{s-1+d}{d},
    \left\lvert H_{d,\lambda,S} \cap \supp{\widehat{\chi}} \right\rvert \right\} \\
    &\qquad + \sum_{n=2}^{n_\mathrm{d}}
    \sum_{(\lambda_1,\ldots,\lambda_n) \in \Lambda^n}
    \left\lvert \left\langle \supp{\widehat{g_{\lambda_n}}} \right\rangle_d
    \cap \supp{\widehat{\chi}}\right\rvert\Biggr).
\end{align*} 
\end{theorem}
\begin{proof}
\sloppy The root node outputs $f\ast\chi$, $f \in E_S$, and thus contributes at most $\lvert (\bigcup_{k \in S} \suppinline{\widehat{\xi_k}}) \cap \suppinline{\widehat{\chi}}\rvert$. 
For the nodes of the first layer, one can directly apply \cref{L:UP_one_node}. 
For a node of depth $n \geq 2$, the input $U[\lambda_1,\ldots,\lambda_{n-1}]f$ need not lie in $E_S$, but trivially lies in $\mathbb C^M$. 
Particularizing the proof of \cref{L:UP_one_node} to the case where the input space is $\mathbb C^M$ yields
\begin{align*}
    \dimC{\spanC{\left\{\rho(u \ast g_{\lambda_n})\ast \chi \colon u \in \mathbb C^M\right\}}} \leq \left\lvert \langle \supp{\widehat{g_{\lambda_n}}} \rangle_d \cap \supp{\widehat{\chi}}\right\rvert.
\end{align*}
Summing these estimates over all nodes of the scattering tree and applying \cref{prop:sc-complex} establishes the desired bound. 
\end{proof}
\sloppy To maximize this upper bound, the filters $\{g_\lambda\}_{\lambda \in \Lambda}$ should be chosen such that $H_{d,\lambda,S} = \mathbb{Z}/M\mathbb{Z}$. For $d$ sufficiently large, this is the case if and only if $H_{\lambda,S}$ is not contained in a coset of a proper subgroup of $\mathbb{Z}/M\mathbb{Z}$. Indeed, for $h_0 \in H_{\lambda,S}$, let $B \coloneqq -h_0+ H_{\lambda,S}$. Then, by commutativity, $H_{d,\lambda,S} = dh_0+\langle B \rangle_d$. As the identity element of $\mathbb{Z}/M\mathbb{Z}$ lies in $B$, we have $\langle B \rangle_\ell \subseteq \langle B \rangle_{\ell+1}$, for all $\ell \in \mathbb{N}$. Thus, there exists a $d_0 \in \mathbb{N}$ with $\langle B \rangle_\ell = \langle B \rangle_{d_0}$, for all $\ell \geq d_0$. Now, $\langle B \rangle_{d_0}+ \langle B \rangle_{d_0} = \langle B \rangle_{2d_0} = \langle B \rangle_{d_0}$, which shows that $\langle B \rangle_{d_0}$ is closed under the group operation and thus, being a nonempty finite subset of $\mathbb{Z}/M\mathbb{Z}$, a subgroup \cite[Theorem 3.3]{gallian2010contemporary}. As $\langle B \rangle_{d_0}$ contains $B$ and is contained in the subgroup $\langle B \rangle$ generated by $B$, it equals $\langle B \rangle$. Consequently, $H_{d,\lambda,S} = dh_0+ \langle B \rangle$, whenever $d \geq d_0$.
It remains to verify that $\langle B \rangle = \mathbb{Z}/M\mathbb{Z}$ if and only if $H_{\lambda,S}$ is not contained in a coset of a proper subgroup of $\mathbb{Z}/M\mathbb{Z}$. If $\langle B \rangle$ is proper, then $H_{\lambda,S} = h_0+ B \subseteq h_0 +\langle B \rangle$. Conversely, if $H_{\lambda,S} \subseteq a+K$ for some $a \in \mathbb{Z}/M\mathbb{Z}$ and some proper subgroup $K$, then $h_0 \in a+K$ implies $a+K = h_0+ K$, so that $B = -h_0+H_{\lambda,S} \subseteq K$ and $\langle B \rangle \subseteq K$. 

\paragraph{Lower bound on $\sepcaps{2s}{\Phi}$} For the lower bound, we follow the approach presented in \cref{cor:lower-bound-real-analytic}. 
To this end, for $S\subseteq \mathcal{K}$ with $\lvert S \rvert =s$, let $\{K_{S,j}\}_{j \in \mathcal{J}_{S}} \subseteq \mathbb R ^{2s}$ be a countable sequence of compact sets such that $\{\psi_{S,j}\colon K_{S,j}\to E_{S,j}\}_{j \in \mathcal{J}_S, S\subseteq \mathcal{K}, \lvert S \rvert =s}$ constitutes a bi-Lipschitz parametrization for $E$ in the sense of \cref{L:bi-Lip}, with $E_{S,j} = \psi_{S,j}(K_{S,j})$. 
By carrying out the same computation as in \cref{eq:inner-prod-real-complex}, one can deduce that 
\begin{align*}
    \left\lVert \widetilde{\Phi}\left(\begin{pmatrix}
        \Re(f) \\ \Im(f)
    \end{pmatrix}\right)\right\rVert &= \left\lVert {\Phi}(f)\right\rVert, \quad f \in E, \\
    \innerprod{\widetilde{\Phi}\left(\begin{pmatrix}
        \Re(f_1) \\ \Im(f_1)
    \end{pmatrix}\right)}{\widetilde{\Phi}\left(\begin{pmatrix}
        \Re(f_2) \\ \Im(f_2)
    \end{pmatrix}\right)} &= \Re(\innerprod{\Phi(f_1)}{\Phi(f_2)}), \quad f_1,f_2 \in E.
\end{align*}
Consequently, for complex-valued $\Phi$ on $s$-sparse signals in $\mathbb C^M$, \cref{cor:lower-bound-real-analytic} takes the form
\begin{align*}
    \sepcaps{2s}{\Phi} &\geq \inf_{\substack{S \subseteq \mathcal{K} \\ \lvert S \rvert =s}} \inf_{j \in \mathcal{J}_S} \frac{2 \left(\int_{E_{S,j}} \left\lVert \Phi(f)\right\rVert^2\,\mathrm{d}\mathcal{H}^{2s}(f)\right)^2}{\int_{E_{S,j}}\int_{E_{S,j}} \left\lvert \Re(\innerprod{\Phi(f_1)}{\Phi(f_2)})\right\rvert^2 \,\mathrm{d}\mathcal{H}^{2s}(f_1)\mathrm{d}\mathcal{H}^{2s}(f_2)} \\
    &\geq \inf_{\substack{S \subseteq \mathcal{K} \\ \lvert S \rvert =s}} \inf_{j \in \mathcal{J}_S} \frac{2 \left(\int_{E_{S,j}} \left\lVert \Phi(f)\right\rVert^2\,\mathrm{d}\mathcal{H}^{2s}(f)\right)^2}{\int_{E_{S,j}}\int_{E_{S,j}} \left\lvert \innerprod{\Phi(f_1)}{\Phi(f_2)}\right\rvert^2 \,\mathrm{d}\mathcal{H}^{2s}(f_1)\mathrm{d}\mathcal{H}^{2s}(f_2)}. 
\end{align*}

\sloppy In what follows, we first analyze the RHS for the unfiltered first-layer map $\Phi_1(f)=(\rho(f \ast g_\lambda))_{\lambda \in \Lambda}$ and subsequently extend
the analysis to the first-layer feature map
$\Phi^1(f) = (\rho(f \ast g_\lambda)\ast \chi)_{\lambda \in \Lambda}$, $f \in E$.
To state our result, let us introduce the matrices $A_{\Lambda} \in \mathbb C^{\binom{M-1+d}{d} \times \binom{M-1+d}{d}}$ and $G_{S,j}\in \mathbb C^{\binom{M-1+d}{d} \times \binom{M-1+d}{d}}$ which are characterized by the filters $\{g_{\lambda}\}_{\lambda \in \Lambda}$ and the geometry of $E_{S,j}$, respectively. Specifically, 
\begin{align*}
    G_{S,j} \coloneqq \int_{E_{S,j}} v_{M,d}(\widehat{f})v_{M,d}(\widehat{f})^{\mathsf{H}}\,\mathrm{d}\mathcal{H}^{2s}(f),
\end{align*}
and $A_{\Lambda} \coloneqq \frac{1}{M}\sum_{\lambda \in \Lambda} A_{\lambda}^{\mathsf{H}}A_{\lambda}$, where $A_\lambda \in \mathbb C^{M \times \binom{M-1+d}{d}}$ with entries
\begin{align*}
    (A_{\lambda})_{k \alpha} = \begin{cases}
            M^{1-d}\binom{d}{\alpha} \widehat{g_\lambda}^\alpha, &\text{if $\sum_{t=0}^{M-1}t \alpha_t \equiv k \Mod{M}$}, \\
            0, &\text{otherwise},
    \end{cases}
\end{align*}
for all $k\in\{0,\ldots,M-1\}$ and $\alpha \in \mathbb N_0^M$ with $\lvert \alpha \rvert =d$. For any such $\alpha$, let $\gamma_{S,j,\alpha}^{\mathsf{H}}$ denote the corresponding row of $G_{S,j}$ and $a_{\Lambda,\alpha}$ denote the corresponding column of $A_{\Lambda}$. 

\begin{lemma}\label{L:LB-1.layer}
    For $\Phi_1(f)=(\rho(f \ast g_\lambda))_{\lambda \in \Lambda}$, $f \in E$, we have
    \begin{align*}
        \sepcaps{2s}{\Phi_1} \geq \inf_{\substack{S \subseteq \mathcal{K} \\ \lvert S \rvert =s}} \inf_{j \in \mathcal{J}_S} \frac{2\left(\trace{G_{S,j}A_\Lambda }\right)^2}{\trace{(G_{S,j}A_{\Lambda})^2}}=\inf_{\substack{S \subseteq \mathcal{K} \\ \lvert S \rvert =s}} \inf_{j \in \mathcal{J}_S} \frac{2\left(\sum_{\alpha} \innerprod{\gamma_{S,j,\alpha}}{a_{\Lambda,\alpha}}\right)^2}{\sum_{\alpha,\beta} \innerprod{\gamma_{S,j,\alpha}}{a_{\Lambda,\beta}} \innerprod{\gamma_{S,j,\beta}}{a_{\Lambda,\alpha}}}. 
    \end{align*}
\end{lemma}
\begin{proof}
    The DFT matrix $F_M \in \mathbb C^{M \times M}$ has inverse $\frac 1M F_M^{\mathsf{H}}$, so that for $f_1,f_2 \in E$, 
    \begin{align*}
        \innerprod{\Phi_1(f_1)}{\Phi_1(f_2)} = \sum_{\lambda \in \Lambda} \innerprod{\rho(f_1\ast g_\lambda)}{\rho(f_2\ast g_\lambda)} = \frac{1}{M}\sum_{\lambda \in \Lambda} \innerprod{F_M(\rho(f_1\ast g_\lambda))}{F_M(\rho(f_2\ast g_\lambda))}. 
    \end{align*}
    As derived in the proof of \cref{L:UP_one_node}, we can write\footnote{In comparison to the matrix $A$ defined in the proof of \cref{L:UP_one_node}, we drop the condition $\suppinline{\alpha} \subseteq H_{\lambda,S}$ here, so that $A_\lambda$ depends only on $g_\lambda$.} $F_M(\rho(f \ast g_{\lambda})) = A_{\lambda}v_{M,d}(\widehat{f})$, $f \in E$.
    Consequently,
    \begin{align*}
        \innerprod{\Phi_1(f_1)}{\Phi_1(f_2)} &= \frac{1}{M}\sum_{\lambda\in \Lambda} \innerprod{A_{\lambda}v_{M,d}(\widehat{f_1})}{A_{\lambda}v_{M,d}(\widehat{f_2})} \\
        &= \innerprod{\frac{1}{M}\sum_{\lambda\in \Lambda}A_{\lambda}^{\mathsf{H}}A_{\lambda}v_{M,d}(\widehat{f_1})}{v_{M,d}(\widehat{f_2})} \\
        &= \innerprod{A_{\Lambda}v_{M,d}(\widehat{f_1})}{v_{M,d}(\widehat{f_2})}. 
    \end{align*}
    This leads to
    \allowdisplaybreaks
    \begin{align}
        &\int_{E_{S,j}}\int_{E_{S,j}} \left\lvert \innerprod{\Phi_1(f_1)}{\Phi_1(f_2)}\right\rvert^2 \,\mathrm{d}\mathcal{H}^{2s}(f_1)\mathrm{d}\mathcal{H}^{2s}(f_2) \nonumber\\
        &\quad= \int_{E_{S,j}}\int_{E_{S,j}} \left\lvert \innerprod{A_{\Lambda}v_{M,d}(\widehat{f_1})}{v_{M,d}(\widehat{f_2})}\right\rvert^2 \,\mathrm{d}\mathcal{H}^{2s}(f_1)\mathrm{d}\mathcal{H}^{2s}(f_2) \nonumber\\
        &\quad= \int_{E_{S,j}}\int_{E_{S,j}} v_{M,d}(\widehat{f_2})^{\mathsf{H}} A_{\Lambda} v_{M,d}(\widehat{f_1}) v_{M,d}(\widehat{f_1})^{\mathsf{H}} A_{\Lambda} v_{M,d}(\widehat{f_2}) \,\mathrm{d}\mathcal{H}^{2s}(f_1)\mathrm{d}\mathcal{H}^{2s}(f_2) \nonumber\\
        &\quad= \int_{E_{S,j}}\int_{E_{S,j}} \trace{A_{\Lambda} v_{M,d}(\widehat{f_1}) v_{M,d}(\widehat{f_1})^{\mathsf{H}} A_{\Lambda} v_{M,d}(\widehat{f_2}) v_{M,d}(\widehat{f_2})^{\mathsf{H}}} \,\mathrm{d}\mathcal{H}^{2s}(f_1)\mathrm{d}\mathcal{H}^{2s}(f_2) \nonumber\\
        &\quad = \trace{A_{\Lambda} \int_{E_{S,j}} v_{M,d}(\widehat{f_1}) v_{M,d}(\widehat{f_1})^{\mathsf{H}} \,\mathrm{d}\mathcal{H}^{2s}(f_1) A_{\Lambda} \int_{E_{S,j}} v_{M,d}(\widehat{f_2}) v_{M,d}(\widehat{f_2})^{\mathsf{H}} \,\mathrm{d}\mathcal{H}^{2s}(f_2)}\nonumber\\
        &\quad = \trace{A_{\Lambda} G_{S,j} A_{\Lambda} G_{S,j}} \nonumber\\
        &\quad = \trace{(G_{S,j}A_{\Lambda})^2} \nonumber\\
        &\quad= \sum_{\alpha,\beta} (G_{S,j}A_{\Lambda})_{\alpha,\beta} (G_{S,j}A_{\Lambda})_{\beta,\alpha} \nonumber\\
        &\quad= \sum_{\alpha,\beta} \innerprod{\gamma_{S,j,\alpha}}{a_{\Lambda,\beta}} \innerprod{\gamma_{S,j,\beta}}{a_{\Lambda,\alpha}}.
        \label{eq:LB-1layer-denominator}
    \end{align}
    Likewise, we have
    \allowdisplaybreaks
    \begin{align}
        \int_{E_{S,j}} \left\lVert \Phi_1(f)\right\rVert^2\,\mathrm{d}\mathcal{H}^{2s}(f) &= \int_{E_{S,j}} \innerprod{A_{\Lambda}v_{M,d}(\widehat{f})}{v_{M,d}(\widehat{f})}\,\mathrm{d}\mathcal{H}^{2s}(f) \nonumber\\
        &= \int_{E_{S,j}} \trace{A_{\Lambda} v_{M,d}(\widehat{f}) v_{M,d}(\widehat{f})^{\mathsf{H}}} \,\mathrm{d}\mathcal{H}^{2s}(f) \nonumber\\
        &= \trace{A_{\Lambda} \int_{E_{S,j}} v_{M,d}(\widehat{f}) v_{M,d}(\widehat{f})^{\mathsf{H}} \,\mathrm{d}\mathcal{H}^{2s}(f)}\nonumber\\
        &= \trace{A_{\Lambda} G_{S,j}} \nonumber\\
        &= \sum_{\alpha} (G_{S,j}A_{\Lambda})_{\alpha,\alpha} \nonumber\\
        &= \sum_{\alpha} \innerprod{\gamma_{S,j,\alpha}}{a_{\Lambda,\alpha}}. \label{eq:LB-1layer-numerator}
    \end{align}
    Combining \cref{eq:LB-1layer-denominator,eq:LB-1layer-numerator} gives the desired bound.
\end{proof}

The trace ratio is maximized when, for all $j \in \mathcal{J}_S$ and all $S \subseteq \mathcal{K}$ with $\lvert S \rvert = s$, the systems $\{\gamma_{S,j,\alpha}\}_{\alpha}$ and $\{a_{\Lambda,\alpha}\}_{\alpha}$ are biorthogonal, i.e.,
$\innerprod{\gamma_{S,j,\alpha}}{a_{\Lambda,\beta}} = \mathbbm{1}_{\{\alpha=\beta\}}$, for all $\alpha,\beta$. Exact biorthogonality would make the lower bound equal to $2\binom{M-1+d}{d}$, but the structure of $A_\Lambda$ generally prevents this. We therefore seek approximate biorthogonality, quantified through the restricted isometry property in \cref{app:cs}.   
Another drawback of directly inferring design choices for $\{g_\lambda\}_{\lambda \in \Lambda}$ from \cref{L:LB-1.layer} is that it is not readily apparent how $G_{S,j}$ depends on $\Xi$. We, however, aim to establish a design criterion for $\{g_\lambda\}_{\lambda \in \Lambda}$ in terms of $\Xi$. To this end, consider the following setup. 

\sloppy Assume from now on that $\lvert \mathcal{K} \rvert < \infty$. With the labeling $\mathcal{K} = \{k_1, \ldots, k_{|\mathcal{K}|}\}$, define the matrix $\Xi \in \mathbb{C}^{M \times |\mathcal{K}|}$ with columns $(\xi_{k_1}, \ldots, \xi_{k_{|\mathcal{K}|}})$, and set $\widehat \Xi = F_M \Xi$. Let $\iota_S \in \{0,1\}^{|\mathcal{K}| \times s}$ be the matrix with $(\iota_S)_{\ell m} \coloneqq \mathbbm{1}_{\{k_\ell = k_{S,m}\}}$ for $\ell \in \{1,\ldots,|\mathcal{K}|\}$ and $m \in \{1,\ldots,s\}$, so that $\Xi_S \coloneqq \Xi\iota_S \in \mathbb{C}^{M \times s}$ is the submatrix of $\Xi$ with columns $(\xi_{k_{S,1}}, \ldots, \xi_{k_{S,s}})$. Likewise, we set $\widehat{\Xi}_S \coloneqq F_M \Xi_S$.
The bi-Lipschitz parametrization $\{\psi_{S,j}\colon K_{S,j} \to E_{S,j}\}_{j \in \mathcal{J}_S, S\subseteq \mathcal{K}, \lvert S\rvert=s}$ can be chosen such that $\psi_{S,j}(x) = \Xi_S (x'+ix'')$, where $x = ((x')^{\mathsf{T}}, (x'')^{\mathsf{T}})^{\mathsf{T}} \in K_{S,j}$. 
To isolate the dependence of $G_{S,j}$ on $\Xi$, we leverage the following property of the vector Veronese map $v_{M,d}$. Let $T_{\widehat{\Xi}} \in \mathbb{C}^{\binom{M-1+d}{d} \times \binom{\lvert \mathcal{K} \rvert-1+d}{d}}$ be the unique matrix satisfying
\[
v_{M,d}(\widehat{\Xi} z) = T_{\widehat{\Xi}}\, v_{\lvert \mathcal{K}\rvert,d}(z), \qquad z \in \mathbb{C}^{\lvert \mathcal{K} \rvert}.
\]
Concretely, for multi-indices $\alpha \in \mathbb{N}_0^M$ and $\nu \in \mathbb{N}_0^{\lvert\mathcal{K}\rvert}$ with $|\alpha|=|\nu|=d$, the entry $(T_{\widehat\Xi})_{\alpha\nu}$ is the coefficient of $z^\nu$ in the expansion of $(\widehat\Xi z)^\alpha$.  
Furthermore, we introduce
\begin{align*}
    C_{S,j} \coloneqq \int_{K_{S,j}} v_{s,d}(x'+ix'')v_{s,d}(x'+ix'')^{\mathsf{H}}\,\mathrm{d}x. 
\end{align*}
The matrix $C_{S,j}$ depends on the parametrization domain but not on the frame $\Xi$. 
Denote by $\delta_{s,d}(\Xi;\Lambda)$ the $\binom{s-1+d}{d}$th restricted isometry constant of $(T_{\widehat{\Xi}}^\mathsf{H}A_{\Lambda}T_{\widehat{\Xi}})^{1/2}$, see \cref{def:RIC}. 
\begin{theorem}\label{thm:lb-sparse}
Let $E\subseteq \mathbb C^M$ be the set of $s$-sparse signals, and assume that $\delta_{s,d}(\Xi;\Lambda) <1$. For $\Phi_1(f) = (\rho(f \ast g_\lambda))_{\lambda \in \Lambda}$, $f\in E$, it holds that
\begin{align}\label{eq:lb-sparse}
    \sepcaps{2s}{\Phi_1} \geq \left(\frac{1-\delta_{s,d}(\Xi;\Lambda)}{1+\delta_{s,d}(\Xi;\Lambda)}\right)^2 \min_{\substack{S \subseteq \mathcal{K} \\ \lvert S \rvert =s}}\inf_{j \in \mathcal{J}_S} \frac{2\left(\trace{C_{S,j}}\right)^2}{\trace{\left(C_{S,j}\right)^2}}.
\end{align}
\end{theorem}
\begin{proof}
By the area formula, we have
\begin{align*}
    G_{S,j} &= \int_{K_{S,j}} v_{M,d}(\widehat \Xi_S (x'+ix''))v_{M,d}(\widehat \Xi_S (x'+ix''))^{\mathsf{H}}\mathbf{J}(\Xi_S)\,\mathrm{d}x \\
    &= \mathbf{J}(\Xi_S) T_{\widehat{\Xi}_S} \underbrace{\int_{K_{S,j}} v_{s,d}(x'+ix'')v_{s,d}(x'+ix'')^{\mathsf{H}}\,\mathrm{d}x}_{= C_{S,j}}\; T_{\widehat{\Xi}_S}^\mathsf{H},  
\end{align*}
where $T_{\widehat{\Xi}_S} \in \mathbb{C}^{\binom{M-1+d}{d} \times \binom{s-1+d}{d}}$ is the matrix defined by the relation
\begin{align*}
    v_{M,d}(\widehat{\Xi}_S z) = T_{\widehat{\Xi}_S}\, v_{s,d}(z), \quad \text{for } z \in \mathbb{C}^{s}.
\end{align*}
Using the cyclic property of the trace, we then obtain
\begin{align}
    \frac{\left(\trace{G_{S,j}A_{\Lambda}}\right)^2}{\trace{(G_{S,j}A_{\Lambda} )^2}} &= \frac{\left(\trace{C_{S,j} T_{\widehat{\Xi}_S}^\mathsf{H}A_{\Lambda}T_{\widehat{\Xi}_S}}\right)^2}{\trace{\left(C_{S,j}T_{\widehat{\Xi}_S}^\mathsf{H}A_{\Lambda}T_{\widehat{\Xi}_S}\right)^2}} \nonumber\\
    & \geq \left(\frac{\lambda_{\min}(T_{\widehat{\Xi}_S}^\mathsf{H}A_{\Lambda}T_{\widehat{\Xi}_S})}{\lambda_{\max}(T_{\widehat{\Xi}_S}^\mathsf{H}A_{\Lambda}T_{\widehat{\Xi}_S})}\right)^2 \frac{\left(\trace{C_{S,j}}\right)^2}{\trace{\left(C_{S,j}\right)^2}}, \label{eq:tr-cond-number}
\end{align}
where the inequality holds as $T_{\widehat{\Xi}_S}^\mathsf{H}A_{\Lambda}T_{\widehat{\Xi}_S}$, $C_{S,j}T_{\widehat{\Xi}_S}^\mathsf{H}A_{\Lambda}T_{\widehat{\Xi}_S}C_{S,j}$, and $C_{S,j}$ are positive semi-definite matrices, so that by \cite[Eq. (1)]{fang1994inequalities},
\begin{align*}
      \trace{C_{S,j} T_{\widehat{\Xi}_S}^\mathsf{H}A_{\Lambda}T_{\widehat{\Xi}_S}} &\geq \lambda_{\min}(T_{\widehat{\Xi}_S}^\mathsf{H}A_{\Lambda}T_{\widehat{\Xi}_S}) \trace{C_{S,j}} 
      \intertext{and}
      \trace{C_{S,j} T_{\widehat{\Xi}_S}^\mathsf{H}A_{\Lambda}T_{\widehat{\Xi}_S}C_{S,j} T_{\widehat{\Xi}_S}^\mathsf{H}A_{\Lambda}T_{\widehat{\Xi}_S}} &\leq \lambda_{\max}(T_{\widehat{\Xi}_S}^\mathsf{H}A_{\Lambda}T_{\widehat{\Xi}_S}) \trace{C_{S,j} T_{\widehat{\Xi}_S}^\mathsf{H}A_{\Lambda}T_{\widehat{\Xi}_S}C_{S,j}} \\
      &=\lambda_{\max}(T_{\widehat{\Xi}_S}^\mathsf{H}A_{\Lambda}T_{\widehat{\Xi}_S}) \trace{(C_{S,j})^2 T_{\widehat{\Xi}_S}^\mathsf{H}A_{\Lambda}T_{\widehat{\Xi}_S}} \\
      &\leq \left(\lambda_{\max}(T_{\widehat{\Xi}_S}^\mathsf{H}A_{\Lambda}T_{\widehat{\Xi}_S})\right)^2 \trace{(C_{S,j})^2}. 
\end{align*}
Here, $\lambda_{\min}(T_{\widehat{\Xi}_S}^\mathsf{H}A_{\Lambda}T_{\widehat{\Xi}_S})$ and $\lambda_{\max}(T_{\widehat{\Xi}_S}^\mathsf{H}A_{\Lambda}T_{\widehat{\Xi}_S})$ denote the smallest and largest eigenvalue of $T_{\widehat{\Xi}_S}^\mathsf{H}A_{\Lambda}T_{\widehat{\Xi}_S}$, respectively. 
Consequently, by \cref{L:LB-1.layer},
\begin{align}\label{eq:lb-cond-sparse}
    \sepcaps{2s}{\Phi_1} \geq \min_{\substack{S \subseteq \mathcal{K} \\ \lvert S \rvert =s}}\left(\frac{\lambda_{\min}(T_{\widehat{\Xi}_S}^\mathsf{H}A_{\Lambda}T_{\widehat{\Xi}_S})}{\lambda_{\max}(T_{\widehat{\Xi}_S}^\mathsf{H}A_{\Lambda}T_{\widehat{\Xi}_S})}\right)^2 \inf_{j \in \mathcal{J}_S} \frac{2\left(\trace{C_{S,j}}\right)^2}{\trace{\left(C_{S,j}\right)^2}}. 
\end{align}
We use $(T_{\widehat{\Xi}}^\mathsf{H}A_{\Lambda}T_{\widehat{\Xi}})^{1/2}_{\mathscr{S}}$ to denote the submatrix of $(T_{\widehat{\Xi}}^\mathsf{H}A_{\Lambda}T_{\widehat{\Xi}})^{1/2}$ with columns indexed by some set $\mathscr{S}$.  
By \cref{L:ric-eigen}, we have for all $\mathscr{S}$ with $\vert \mathscr{S}\rvert = \binom{s-1+d}{d}$, 
\begin{align}  
\begin{split}\label{eq:eig-val-ric}
\lambda_{\min}\left(\left((T_{\widehat{\Xi}}^\mathsf{H}A_{\Lambda}T_{\widehat{\Xi}})^{1/2}_{\mathscr{S}}\right)^{\mathsf{H}}(T_{\widehat{\Xi}}^\mathsf{H}A_{\Lambda}T_{\widehat{\Xi}})^{1/2}_{\mathscr{S}}\right) &\geq 1-\delta_{s,d}(\Xi;\Lambda),  \\ \lambda_{\max}\left(\left((T_{\widehat{\Xi}}^\mathsf{H}A_{\Lambda}T_{\widehat{\Xi}})^{1/2}_{\mathscr{S}}\right)^{\mathsf{H}}(T_{\widehat{\Xi}}^\mathsf{H}A_{\Lambda}T_{\widehat{\Xi}})^{1/2}_{\mathscr{S}}\right) &\leq 1+\delta_{s,d}(\Xi;\Lambda).  
\end{split}
\end{align}
As $(T_{\widehat{\Xi}}^\mathsf{H}A_{\Lambda}T_{\widehat{\Xi}})^{1/2}$ is Hermitian, $((T_{\widehat{\Xi}}^\mathsf{H}A_{\Lambda}T_{\widehat{\Xi}})^{1/2}_{\mathscr{S}})^{\mathsf{H}}(T_{\widehat{\Xi}}^\mathsf{H}A_{\Lambda}T_{\widehat{\Xi}})_{\mathscr{S}}^{1/2}$ is the submatrix of $T_{\widehat{\Xi}}^\mathsf{H}A_{\Lambda}T_{\widehat{\Xi}}$ with rows and columns indexed by $\mathscr{S}$. Now, for every $S \subseteq \mathcal{K}$ with $\lvert S \rvert =s$ there is an index set $\mathscr{S}$ with $\lvert \mathscr{S} \rvert = \binom{s-1+d}{d}$ such that $(T_{\widehat{\Xi}})_{\mathscr{S}} = T_{\widehat{\Xi}_S}$. Indeed, for every such $S$, there is a matrix $\mathcal{I}_{S} \in \{0,1\}^{\binom{\lvert \mathcal{K}\rvert-1+d}{d} \times \binom{s-1+d}{d}}$ such that $v_{\lvert \mathcal{K}\rvert,d}(\iota_S z) = \mathcal{I}_{S}\,v_{s,d}(z)$, $z \in \mathbb C^s$. Therefore, for every $z \in \mathbb C^s$, we have
\begin{align*}
    v_{M,d}(\widehat{\Xi}_Sz) = v_{M,d}(\widehat{\Xi}\iota_Sz) = T_{\widehat{\Xi}}v_{\lvert \mathcal{K}\rvert,d}(\iota_Sz) = T_{\widehat{\Xi}} \mathcal{I}_{S}\,v_{s,d}(z),
\end{align*}
and thus $T_{\widehat{\Xi}_S} = T_{\widehat{\Xi}} \mathcal{I}_{S} = (T_{\widehat{\Xi}})_{\mathscr{S}}$ for some index set $\mathscr{S}$ with $\lvert \mathscr{S} \rvert = \binom{s-1+d}{d}$. 
Consequently, for every $S$ with $\lvert S\rvert =s$, the matrix $T_{\widehat{\Xi}_S}^\mathsf{H}A_{\Lambda}T_{\widehat{\Xi}_S}$ is the submatrix of $T_{\widehat{\Xi}}^\mathsf{H}A_{\Lambda}T_{\widehat{\Xi}}$ with rows and columns indexed by $\mathscr{S}$. Combining \cref{eq:lb-cond-sparse} with \cref{eq:eig-val-ric} establishes \cref{eq:lb-sparse}. It remains to verify that $\lambda_{\max}(T_{\widehat{\Xi}_S}^{\mathsf{H}} A_{\Lambda} T_{\widehat{\Xi}_S}) > 0$, so that \cref{eq:tr-cond-number,eq:lb-cond-sparse} are well-defined. If, to the
contrary, $\lambda_{\max}(T_{\widehat{\Xi}_S}^{\mathsf{H}} A_{\Lambda}
T_{\widehat{\Xi}_S}) = 0$, then positive semi-definiteness would imply $T_{\widehat{\Xi}_S}^{\mathsf{H}} A_{\Lambda} T_{\widehat{\Xi}_S} = 0$, which, in turn, yields $\delta_{s,d}(\Xi;\Lambda) = 1$, contradicting the assumption $\delta_{s,d}(\Xi;\Lambda) < 1$. This completes the proof.
\end{proof} 

\sloppy The result in \cref{thm:lb-sparse} extends readily to the first-layer feature map $\Phi^1(f) = \bigl(\rho(f\ast g_\lambda)\ast \chi\bigr)_{\lambda \in \Lambda}$, $f \in E$. To this end, set $A_{\lambda,\chi} \coloneqq \diag{\widehat{\chi}}\, A_\lambda$ and $A_{\Lambda,\chi} \coloneqq \frac{1}{M} \sum_{\lambda \in \Lambda} A_{\lambda,\chi}^{\mathsf{H}} A_{\lambda,\chi}$, and denote by $\delta_{s,d}(\Xi;\Lambda,\chi)$ the $\binom{s-1+d}{d}$th restricted isometry constant of $(T_{\widehat{\Xi}}^\mathsf{H}A_{\Lambda,\chi}T_{\widehat{\Xi}})^{1/2}$. 
\begin{corollary}
Let $E\subset \mathbb C^M$ be the set of $s$-sparse signals. Assuming that $\delta_{s,d}(\Xi;\Lambda,\chi)<1$, the first-layer feature map $\Phi^1(f) = (\rho(f \ast g_\lambda)\ast\chi)_{\lambda \in \Lambda}$, $f\in E$, satisfies
\begin{align*}
    \sepcaps{2s}{\Phi^1} \geq \left(\frac{1-\delta_{s,d}(\Xi;\Lambda,\chi)}{1+\delta_{s,d}(\Xi;\Lambda,\chi)}\right)^2 \min_{\substack{S \subseteq \mathcal{K} \\ \lvert S \rvert =s}}\inf_{j \in \mathcal{J}_S} \frac{2\left(\trace{C_{S,j}}\right)^2}{\trace{\left(C_{S,j}\right)^2}}.
\end{align*}    
\end{corollary}
\begin{proof}
    Note that $F_M(\rho(f \ast g_\lambda)\ast\chi) = \diag{\widehat{\chi}} A_\lambda v_{M,d}(\widehat{f})$. Hence the proof of \cref{thm:lb-sparse} carries over verbatim upon replacing $A_\Lambda$ with $A_{\Lambda,\chi}$, which yields the claim.
\end{proof}

\sloppy The lower bound on $\sepcaps{2s}{\Phi^1}$ therefore suggests the following design criterion: choose $\{\chi\}\cup\{g_\lambda\}_{\lambda \in \Lambda}$ so that the $\binom{s-1+d}{d}$th restricted isometry constant of $(T_{\widehat{\Xi}}^\mathsf{H}A_{\Lambda,\chi}T_{\widehat{\Xi}})^{1/2}$ is as small as possible.
 
\subsection{Rectifiable sets}
We now replace the sparse-signal model by a countably $\mathcal{H}^s$-rectifiable set $E \subseteq \mathbb C^M$. We assume that $E$ admits a bi-Lipschitz parametrization $\{\psi_j \colon K_j \to \psi_j(K_j)\}_{j \in \mathcal{J}}$, with $K_j \subset \mathbb R^s$ compact, and that each $\psi_j$ is a polynomial of degree $n_j$, i.e., 
\begin{align*}
\psi_j(x) = \sum_{\alpha \colon \lvert \alpha \rvert \leq n_j} c_{j,\alpha}x^\alpha, \quad x \in K_j, \text{ for some }c_{j,\alpha} \in \mathbb C^M.
\end{align*} 
\paragraph{Upper bound on $\sepcaps{s}{\Phi}$}
For $\Phi\colon E\to \mathbb C^{M'}$ real-analytic, we have by \cite{haberle2026function}, 
\begin{align*}
    \sepcaps{s}{\Phi} = \sepcaps{s}{\widetilde{\Phi}} = \min_{j \in \mathcal{J}}
    \sepcap{\widetilde\Phi \circ \widetilde{\psi}_j} = 2\min_{j \in \mathcal{J}}
    \dimR{\spanR{(\widetilde{\Phi}\circ \widetilde{\psi}_j)(K_j)}},
\end{align*}
where $\widetilde{\psi}_j \coloneqq \iota \circ \psi_j$.
By the same approach as in the derivation of \cref{eq:sc-sparse-3}, we obtain
\begin{align}\label{eq:dimspanC-rect}
    \sepcaps{s}{\Phi} = 2\min_{j \in \mathcal{J}} \dimC{\spanC{\left\{\begin{pmatrix}
        \varphi^\mathsf{T},\overline{\varphi}^\mathsf{T}
    \end{pmatrix}^\mathsf{T}\colon \varphi \in \Phi(E_j)\right\}}}, 
\end{align}
where $E_j \coloneqq \psi_j(K_j)$.

To develop an upper bound on $\sepcaps{s}{\Phi}$, we again start with a one-node estimate. Let $H_{\lambda,\psi_j} \coloneqq (\bigcup_{x \in K_j} \suppinline{\widehat{\psi_j(x)}}) \cap \supp{\widehat{g_\lambda}}$, and define 
\begin{align*}
    H_{d,\lambda,\psi_j} = \underbrace{H_{\lambda,\psi_j}+ \cdots +H_{\lambda,\psi_j}}_{d \text{ times}}.
\end{align*}
\begin{lemma}\label{L:node-rect}
    Suppose that $\psi_j \colon K_j \to E_j$ is a polynomial of degree $n_j$, $j \in \mathcal{J}$. Then, for each $j \in \mathcal{J}$, 
    \begin{align}
        \dimC{\spanC{\left\{\rho(f\ast g_\lambda) \colon f \in E_j\right\}}} &\leq \min\left\{\binom{s+n_jd}{n_jd}, \lvert H_{d,\lambda,\psi_j} \rvert\right\} \label{eq:node_rect1}\\
        \intertext{and}
        \dimC{\spanC{\left\{\rho(f\ast g_\lambda)\ast \chi \colon f \in E_j\right\}}} &\leq \min\left\{\binom{s+n_jd}{n_jd}, \lvert H_{d,\lambda,\psi_j} \cap \supp{\widehat{\chi}}\rvert\right\}. \label{eq:node_rect2}
    \end{align}
\end{lemma}
\begin{proof}
    As in the proof of \cref{L:UP_one_node}, we take the DFT and obtain
    \begin{align*}
        F_M(\rho(f \ast g_\lambda)) = A v_{M,d}(\widehat{f}), \quad f \in E_j,
    \end{align*}
    where $A \in \mathbb C^{M \times \binom{M-1+d}{d}}$ with entries
    \begin{align*}
        A_{k\alpha} = \begin{cases}
            M^{1-d} \binom{d}{\alpha} \widehat{g_\lambda}^\alpha, &\text{if $\sum_{t=0}^{M-1}t \alpha_t \equiv k \Mod{M}$ and $\suppinline{\alpha} \subseteq H_{\lambda,\psi_j}$,} \\
            0, &\text{otherwise,}
        \end{cases}
    \end{align*}
    for $k \in \{0,\dots,M-1\}$ and $\alpha \in \mathbb N_0^M$ with $\lvert \alpha \rvert =d$ in degree lexicographic order. Then, by the same argument as in the proof of \cref{L:UP_one_node}, $\rank{A} = \lvert H_{d,\lambda,\psi_j} \rvert$. Furthermore, it holds that
    \begin{align*}
        \dimC{\spanC{v_{M,d}(F_M(E_j))}} \leq \sum_{\ell = 0}^{n_jd}\binom{s-1+\ell}{\ell} = \binom{s+n_jd}{n_jd},
    \end{align*}
    as $x\mapsto v_{M,d}(F_M(\psi_j(x)))$ is a multivariate polynomial of degree at most $n_jd$. This establishes \cref{eq:node_rect1}. The assertion in \cref{eq:node_rect2} holds because $F_M(\rho(f\ast g_\lambda)\ast \chi) = \diag{\widehat{\chi}}A v_{M,d}(\widehat{f})$, so that $\rank{\diag{\widehat{\chi}}A} = \lvert H_{d,\lambda,\psi_j} \cap \supp{\widehat{\chi}}\rvert$.
\end{proof}
\begin{theorem}
    For the scattering network $\Phi$ of depth $n_\mathrm{d}$, we have
\begin{align*}
    \sepcaps{s}{\Phi} &\leq 4 \min_{j \in \mathcal{J}} \left(\min\left\{\binom{s+n_j}{n_j},
    \left\lvert \left(\bigcup_{x \in K_j} \supp{\widehat{\psi_j(x)}}\right)
    \cap \supp{\widehat{\chi}} \right\rvert \right\}
    \vphantom{\sum_{n=1}^{n_\mathrm{d}}\sum_{(p,\lambda_n) \in \Lambda^{n-1}\times\Lambda}  \left(\min\left\{\binom{s+n_jd^n}{n_jd^n},\left\lvert H_{d,\lambda_n,U[p]\circ\psi_j} \cap \supp{\widehat{\chi}} \right\rvert \right\}\right)}\right. \\&\qquad \left.+ \sum_{n=1}^{n_\mathrm{d}}\sum_{(p,\lambda_n) \in \Lambda^{n-1}\times\Lambda}  \left(\min\left\{\binom{s+n_jd^n}{n_jd^n},\left\lvert H_{d,\lambda_n,U[p]\circ\psi_j} \cap \supp{\widehat{\chi}} \right\rvert \right\}\right)\right).
\end{align*} 
\end{theorem}
\begin{proof}
    \sloppy The input of the node indexed by $(p,\lambda_n) \in \Lambda^{n-1}\times\Lambda$ is parametrized by $U[p]\circ\psi_j$, which is a polynomial of degree at most $n_j d^{\,n-1}$. Indeed, $U[p]$ consists of $n-1$ node operations \cref{eq:operation-node}, each of which multiplies the degree by at most $d$. 
    Note that the proof of \cref{L:node-rect} directly applies to $U[p]\circ\psi_j$ in place of $\psi_j$ and bounds the span dimension of this node's outputs by $\min\{\binom{s+n_jd^{\,n}}{n_jd^{\,n}}, \lvert H_{d,\lambda_n,\,U[p]\circ\psi_j} \cap \suppinline{\widehat{\chi}}\rvert\}$.
    The output of the root node takes the form $\psi_j(x)\ast\chi$, for $x \in K_j$, which is a polynomial of degree at most $n_j$ in $x$ and whose Fourier transform is supported in $(\bigcup_{x \in K_j} \suppinline{\widehat{\psi_j(x)}}) \cap \suppinline{\widehat{\chi}}$. Hence it contributes at most $\min\{\binom{s+n_j}{n_j},\lvert (\bigcup_{x \in K_j} \suppinline{\widehat{\psi_j(x)}}) \cap \suppinline{\widehat{\chi}} \rvert\}$. 
    Summing over all nodes of the tree and applying \cref{eq:dimspanC-rect} yields the desired bound.
\end{proof}
The nonlinear nature of $E$ is captured in this upper bound by the degrees of the maps $\psi_j$ and the sets $H_{d,\lambda_n,U[p]\circ\psi_j}$. Both a higher degree
of $\psi_j$ and larger sets $H_{d,\lambda_n,U[p]\circ\psi_j}$ reflect a richer geometry, and both raise the upper bound. As the degree of $U[p]\circ\psi_j$ grows like $n_jd^{n-1}$ along the tree, both effects are amplified with depth.

\paragraph{Lower bound on $\sepcaps{s}{\Phi}$} 
\sloppy \cref{cor:lower-bound-real-analytic}, together with the bound
$\lvert \Re(\innerprod{\Phi(f_1)}{\Phi(f_2)}) \rvert \leq \lvert \innerprod{\Phi(f_1)}{\Phi(f_2)} \rvert$, yields
\begin{align*}
    \sepcaps{s}{\Phi} \geq \inf_{j \in \mathcal{J}}
    \frac{2 \left(\int_{E_{j}} \left\lVert \Phi(f)\right\rVert^2\,\mathrm{d}\mathcal{H}^{s}(f)\right)^2}
    {\int_{E_{j}}\int_{E_{j}} \left\lvert \innerprod{\Phi(f_1)}{\Phi(f_2)}\right\rvert^2 \,\mathrm{d}\mathcal{H}^{s}(f_1)\,\mathrm{d}\mathcal{H}^{s}(f_2)}.
\end{align*}
We first consider again the unfiltered first-layer map $\Phi_1(f)=(\rho(f \ast g_\lambda))_{\lambda \in \Lambda}$, $f\in E$. With 
\begin{align*}
    G_j \coloneqq \int_{E_j} v_{M,d}(\widehat{f})v_{M,d}(\widehat{f})^{\mathsf{H}} \,\mathrm{d}\mathcal{H}^s(f), 
\end{align*}
the same computation as in the proof of \cref{L:LB-1.layer} particularizes the lower
bound for $\Phi_1$ to 
\begin{align}\label{eq:lb-rect}
    \sepcaps{s}{\Phi_1} \geq \inf_{j \in \mathcal{J}} \frac{2 (\trace{G_j A_\Lambda})^2}{\trace{(G_j A_\Lambda)^2}}.
\end{align}

As in the sparse case, it is not immediate how $G_j$ depends on the geometry of $E$. To isolate this dependence, let $T_{\widehat{\psi}_j}$ be the unique matrix satisfying
\begin{align*}
    v_{M,d}(\widehat{\psi_j(x)}) = T_{\widehat{\psi}_j} w_{s,n_jd}(x),
\end{align*}
where $w_{s,n_jd}(x) = (x^\alpha)_{\lvert \alpha\rvert \leq n_jd}$, $x \in \mathbb R^s$.
In particular, $T_{\widehat{\psi}_j}$ depends on the coefficients $(c_{j,\alpha})_{|\alpha|\leq n_j} \subset \mathbb{C}^M$ of the parametrization $\psi_j$. We further set
\begin{align*}
    M_j \coloneqq \int_{K_j} w_{s,n_jd}(x)\, w_{s,n_jd}(x)^{\mathsf{H}} \, \mathrm{d}x,
\end{align*}
which, unlike $T_{\widehat{\psi}_j}$, depends only on the domain $K_j$ of the parametrization $\psi_j$ and not on its coefficients. Let $\lambda_{\min}(T_{\widehat{\psi}_j}^\mathsf{H} A_\Lambda T_{\widehat{\psi}_j})$ and $\lambda_{\max}(T_{\widehat{\psi}_j}^\mathsf{H} A_\Lambda T_{\widehat{\psi}_j})$ denote the smallest and largest eigenvalue of $T_{\widehat{\psi}_j}^\mathsf{H} A_\Lambda T_{\widehat{\psi}_j}$, respectively. 
\begin{theorem}\label{thm:lb-rect}
\sloppy Let $E \subseteq \mathbb C^M$ be countably $\mathcal{H}^s$-rectifiable with bi-Lipschitz parametrization $\{\psi_j \colon K_j \to E_j\}_{j \in \mathcal{J}}$, where each $\psi_j$ is $\lipbound$-bi-Lipschitz. Suppose that $\psi_j \colon K_j \to E_j$ is a polynomial of degree $n_j$, $j \in \mathcal{J}$, and consider $\Phi_1(f) = (\rho(f \ast g_\lambda))_{\lambda \in \Lambda}$, $f\in E$.
\begin{enumerate}[label=(\alph*)]
    \item If $\Phi_1$ vanishes $\mathcal{H}^s$-a.e.\ on $E_j$, for some
    $j\in\mathcal{J}$, then $\sepcaps{s}{\Phi_1}=0$.
    \item If $\Phi_1$ does not vanish $\mathcal{H}^s$-a.e.\ on $E_j$, for any
    $j\in\mathcal{J}$, then
    \begin{align*}
    \sepcaps{s}{\Phi_1} \geq \lipbound^{-4s} \inf_{j \in \mathcal{J}} \left(\frac{\lambda_{\min}\left(T_{\widehat{\psi}_j}^\mathsf{H} A_\Lambda T_{\widehat{\psi}_j}\right)}{\lambda_{\max}\left(T_{\widehat{\psi}_j}^\mathsf{H} A_\Lambda T_{\widehat{\psi}_j}\right)}\right)^2 \frac{2(\trace{M_j})^2}{\trace{(M_j)^2}}.
\end{align*}
\end{enumerate}
\end{theorem}
\begin{proof}
\begin{enumerate}[label=(\alph*)]
    \item The assertion follows from \cref{eq:sepcap-form2}, applied with $A = \{f \in E_j \colon \Phi_1(f) = 0\}$, upon noting that $\mathcal{H}^s(A) = \mathcal{H}^s(E_j) > 0$ and $\dimRinline{\spanRinline{\Phi_1(A)}} = 0$.
    \item Application of the area formula gives 
\begin{align*}
    G_j &= \int_{K_j} v_{M,d}(\widehat{\psi_j(x)})v_{M,d}(\widehat{\psi_j(x)})^{\mathsf{H}} \mathbf{J}(\mathrm{d}(\psi_j)_x)\,\mathrm{d}x \\
    &= T_{\widehat{\psi}_j} \int_{K_j} w_{s,n_jd}(x)w_{s,n_jd}(x)^{\mathsf{H}} \mathbf{J}(\mathrm{d}(\psi_j)_x)\,\mathrm{d}x\; T_{\widehat{\psi}_j}^\mathsf{H}. 
\end{align*}
As $\psi_j\colon K_j \to E_j$ is $\lipbound$-bi-Lipschitz, we have $\lipbound^{-s} \leq \mathbf{J}(\mathrm{d}(\psi_j)_x) \leq \lipbound^s$, for $\mathcal{L}^s$-a.e. $x\in K_j$ and $j \in \mathcal{J}$. Thus,
\begin{align*}
    \lipbound^{-s} T_{\widehat{\psi}_j} M_j T_{\widehat{\psi}_j}^\mathsf{H} \preceq G_j \preceq \lipbound^s T_{\widehat{\psi}_j} M_j T_{\widehat{\psi}_j}^\mathsf{H}. 
\end{align*}
As $A_\Lambda$ is positive semi-definite, we have
\begin{align}
    \trace{G_jA_\Lambda} &\geq \lipbound^{-s} \trace{T_{\widehat{\psi}_j} M_j T_{\widehat{\psi}_j}^\mathsf{H}A_\Lambda} \nonumber\\
    &=\lipbound^{-s} \trace{M_j T_{\widehat{\psi}_j}^\mathsf{H} A_\Lambda T_{\widehat{\psi}_j}}. \label{eq:GA-lb}
\end{align}
Furthermore, since $A_\Lambda G_j A_\Lambda$ and $A_\Lambda T_{\widehat{\psi}_j} M_j T_{\widehat{\psi}_j}^\mathsf{H} A_\Lambda$ are positive semi-definite, we have
\begin{align*}
    \trace{G_jA_\Lambda G_jA_\Lambda} &\leq \lipbound^{s}\trace{T_{\widehat{\psi}_j} M_j T_{\widehat{\psi}_j}^\mathsf{H} A_\Lambda G_jA_\Lambda} \\
    &= \lipbound^{s}\trace{A_\Lambda T_{\widehat{\psi}_j} M_j T_{\widehat{\psi}_j}^\mathsf{H} A_\Lambda G_j} \\
    &\leq \lipbound^{2s}\trace{A_\Lambda T_{\widehat{\psi}_j} M_j T_{\widehat{\psi}_j}^\mathsf{H} A_\Lambda T_{\widehat{\psi}_j} M_j T_{\widehat{\psi}_j}^\mathsf{H}} \\
    &= \lipbound^{2s}\trace{ M_j T_{\widehat{\psi}_j}^\mathsf{H} A_\Lambda T_{\widehat{\psi}_j} M_j T_{\widehat{\psi}_j}^\mathsf{H}A_\Lambda T_{\widehat{\psi}_j}} \\
    &= \lipbound^{2s}\trace{\left(M_j T_{\widehat{\psi}_j}^\mathsf{H} A_\Lambda T_{\widehat{\psi}_j}\right)^2}.
\end{align*}
Thus, 
\begin{align*}
    \frac{\left(\trace{G_jA_\Lambda}\right)^2}{\trace{\left(G_jA_\Lambda\right)^2}} &\geq \frac{\lipbound^{-2s} \left(\trace{M_j T_{\widehat{\psi}_j}^\mathsf{H} A_\Lambda T_{\widehat{\psi}_j}}\right)^2}{\lipbound^{2s}\trace{\left(M_j T_{\widehat{\psi}_j}^\mathsf{H} A_\Lambda T_{\widehat{\psi}_j}\right)^2}}.
\end{align*}
Application of \cref{eq:tr-cond-number} together with \cref{eq:lb-rect} establishes the result. Note that $\lambda_{\max}(T_{\widehat{\psi}_j}^\mathsf{H} A_\Lambda T_{\widehat{\psi}_j}) >0$. Indeed, suppose, for contradiction, that $\lambda_{\max}(T_{\widehat{\psi}_j}^\mathsf{H} A_\Lambda T_{\widehat{\psi}_j}) =0$, then, as $T_{\widehat{\psi}_j}^\mathsf{H} A_\Lambda T_{\widehat{\psi}_j}$ is positive semi-definite, we have $T_{\widehat{\psi}_j}^\mathsf{H} A_\Lambda T_{\widehat{\psi}_j} =0$. This, in turn, implies that $\traceinline{G_j A_{\Lambda}} =0$ upon application of \cref{eq:GA-lb} together with the analogous upper bound $\traceinline{G_jA_\Lambda} \leq \lipbound^{s} \traceinline{M_j T_{\widehat{\psi}_j}^\mathsf{H}A_\Lambda T_{\widehat{\psi}_j} }$. But $\traceinline{G_j A_{\Lambda}} = \int_{E_j} \lVert \Phi_1(f) \rVert^2\,\mathrm{d}\mathcal{H}^s(f)$, so that $\Phi_1$ vanishes $\mathcal{H}^s$-a.e.\ on $E_j$, a contradiction. This completes the proof.  
\end{enumerate}
\end{proof}

Analogously to the sparse-signal model, the result in \cref{thm:lb-rect} extends to the first-layer feature map $\Phi^1$. 
\begin{corollary}
    Under the hypotheses of \cref{thm:lb-rect}, the following statements hold for $\Phi^1(f) = (\rho(f\ast g_\lambda)\ast \chi)_{\lambda \in \Lambda}$, $f \in E$.
    \begin{enumerate}[label=(\alph*)]
    \item If $\Phi^1$ vanishes $\mathcal{H}^s$-a.e.\ on $E_j$, for some
    $j\in\mathcal{J}$, then $\sepcaps{s}{\Phi^1}=0$.
    \item If $\Phi^1$ does not vanish $\mathcal{H}^s$-a.e.\ on $E_j$, for any
    $j\in\mathcal{J}$, then
    \begin{align*}
    \sepcaps{s}{\Phi^1} \geq \lipbound^{-4s} \inf_{j \in \mathcal{J}} \left(\frac{\lambda_{\min}\left(T_{\widehat{\psi}_j}^\mathsf{H} A_{\Lambda,\chi} T_{\widehat{\psi}_j}\right)}{\lambda_{\max}\left(T_{\widehat{\psi}_j}^\mathsf{H} A_{\Lambda,\chi} T_{\widehat{\psi}_j}\right)}\right)^2 \frac{2(\trace{M_j})^2}{\trace{(M_j)^2}}.
\end{align*}
\end{enumerate}
\end{corollary}
\begin{proof}
    As $F_M(\rho(f \ast g_\lambda)\ast\chi) = \diag{\widehat{\chi}}\, A_\lambda\,v_{M,d}(\widehat{f}\,)$, the output-generating atom $\chi$ enters the analysis only through the substitution of $A_{\Lambda,\chi}$ for $A_\Lambda$. Repeating the proof of \cref{thm:lb-rect} with this modification yields the claim.
\end{proof}
We conclude with the design criterion that the filters $\{\chi\}\cup\{g_\lambda\}_{\lambda \in \Lambda}$ should be chosen so that $T_{\widehat{\psi}_j}^\mathsf{H} A_{\Lambda,\chi} T_{\widehat{\psi}_j}$ is well-conditioned, for all $j \in \mathcal{J}$, in the sense that the ratio $\lambda_{\min}(T_{\widehat{\psi}_j}^\mathsf{H} A_{\Lambda,\chi} T_{\widehat{\psi}_j})/\lambda_{\max}(T_{\widehat{\psi}_j}^\mathsf{H} A_{\Lambda,\chi} T_{\widehat{\psi}_j})$ is close to $1$.

\appendix
\section{Restricted isometry property}\label[appendix]{app:cs}
\begin{definition}[$s$th restricted isometry constant, \cite{foucart2013mathematical}]\label{def:RIC}
    The \emph{$s$th restricted isometry constant} $\delta_s(B)$ of a matrix $B \in \mathbb{C}^{m \times N}$ is the smallest $\delta \geq 0$ such that
    \begin{align*}
        (1-\delta)\|x\|^2 \leq \|Bx\|^2 \leq (1 + \delta)\|x\|^2,
    \end{align*}
    for all $x \in \mathbb{C}^N$ with $\lvert \supp{x} \rvert \leq s$.
\end{definition}
Denote by $B_\tau$ the $(m \times \lvert \tau\rvert)$-submatrix of $B$ with columns indexed by some set $\tau$, and let $\lambda_{\min}(B_\tau ^\mathsf{H}B_\tau)$ and $\lambda_{\max}(B_\tau ^\mathsf{H}B_\tau)$ be the smallest and largest eigenvalues of $B_\tau ^\mathsf{H}B_\tau$, respectively.  
\begin{lemma}[\cite{candes2005decoding}]\label{L:ric-eigen}
    For $B \in \mathbb{C}^{m \times N}$, and all index sets $\tau$ with $\lvert \tau \rvert \leq s$, it holds that
    \begin{align*}
        1-\delta_s(B)\leq \lambda_{\min}(B_\tau^\mathsf{H}B_\tau) \leq \lambda_{\max}(B_\tau^\mathsf{H}B_\tau) \leq 1+\delta_s(B).
    \end{align*}
\end{lemma}

\bibliography{ref}

\end{document}